\documentclass{article}

\usepackage[preprint]{neurips_2025}

\usepackage[utf8]{inputenc} 
\usepackage[T1]{fontenc}    

\usepackage[colorlinks=true, citecolor=blue]{hyperref} 
\usepackage{url}            
\usepackage{booktabs}       
\usepackage{amsfonts,amsmath,amssymb,amsthm,dsfont,epsfig,graphicx,booktabs,latexsym,mathtools} %
\usepackage{nicefrac}       
\usepackage{microtype}      
\usepackage[dvipsnames]{xcolor}         
\usepackage{wrapfig} 
\usepackage{mdframed}
\usepackage{tikz}
\usepackage{lipsum} 
\usepackage{cleveref} 
\usetikzlibrary{positioning,shadows,arrows,shapes,shapes.arrows,arrows.meta,trees,shapes.misc,shapes.geometric,decorations.pathreplacing,shapes.multipart, fit, backgrounds}
\tikzset{
    -Latex,auto,node distance =1 cm and 1 cm,semithick,
    state/.style ={circle, draw=white, inner sep=0.01cm, minimum width = 0.35 cm},
    snode/.style = {rectangle, draw, inner sep=0.1cm,node contents={}},
    unmeasured/.style = {circle, draw, inner sep=0.05cm,node contents={}},
    point/.style = {circle, draw, inner sep=0.04cm,fill,node contents={}},샤
    bidirected/.style={Latex-Latex,dashed},
    el/.style = {inner sep=2pt, align=left, sloped}
}
\tikzstyle{mybox} = [draw=gray, fill=gray!20, very thick,
rectangle, rounded corners, inner xsep=3pt, inner ysep=7pt]
\usepackage{bm}
\usepackage{enumerate}
\usepackage{cancel}
\usepackage{multirow}
\usepackage{bbm}
\usepackage{color}
\usepackage{array} 
\usepackage{paralist}
\usepackage{verbatim} 
 
\usepackage[noline, ruled, boxed, linesnumbered]{algorithm2e} 

\usepackage[normalem]{ulem}
\usepackage{thmtools}
\usepackage{thm-restate} 
\usepackage{setspace}
\usepackage{adjustbox}
\usepackage{graphicx}
\usepackage{subfig}
\usepackage{clipboard}
\newcommand\labelAndRemember[2]
  {\expandafter\gdef\csname labeled:#1\endcsname{#2}%
   \label{#1}#2}
\newcommand\recallLabel[1]
   {\csname labeled:#1\endcsname\tag{\ref{#1}}}

\usepackage{enumitem,kantlipsum} 
\usepackage{nccmath}
\usepackage{minted}
\usepackage{longtable}

\usepackage{pifont}
\usepackage{sansmath}

\definecolor{opaquegray}{RGB}{192, 192, 192}
\definecolor{afblue}{RGB}{0,0,102}
\definecolor{mypink}{RGB}{255,51,153}

\definecolor{betterpink}{RGB}{218, 94, 156}
\definecolor{betterred}{RGB}{228,26,28}
\definecolor{betterblue}{RGB}{55,126,184}
\definecolor{bettergreen}{RGB}{77,175,74}
\definecolor{betterpurple}{RGB}{152,78,163}
\definecolor{betterorange}{RGB}{255,127,0}
\definecolor{betteryellow}{RGB}{242,177,61} 

\usepackage{tikz}
\usepackage{subcaption}
\usetikzlibrary{positioning, calc, shapes.geometric}

\Crefname{figure}{Fig.}{Figs.}
\Crefname{section}{Sec.}{Secs.}
\Crefname{equation}{Eq.}{Eqs.}
\Crefname{proposition}{Prop.}{Props.}
\Crefname{theorem}{Thm.}{Thms.}
\Crefname{lemma}{Lem.}{Lems}
\Crefname{definition}{Def.}{Defs.}
\Crefname{algorithm}{Alg.}{Algs.}
\Crefname{corollary}{Cor.}{Cors.}

\newcommand{\Gi}[1]{\mathcal{G}_{\overline{#1}}}
\newcommand{\Go}[1]{\mathcal{G}_{\underline{#1}}}

\newcommand{\CI}{\mathrel{\perp\mspace{-10mu}\perp}}

\newcommand{\Pa}{\mathtt{Pa}}

\newcommand{\De}{\mathtt{De}}

\newcommand{\An}{\mathtt{An}}

\newcommand{\Ch}{\mathtt{Ch}}

\newcommand{\bedge}{\leftrightarrow}

\newcommand{\mf}[1]{\mathbf{#1}}
\newcommand{\mc}[1]{\mathcal{#1}}
\newcommand{\mb}[1]{\mathbb{#1}} 

\newcommand{\POMISs}{\mathsf{POMISs}}
\newcommand{\MISs}{\mathsf{MISs}}
\newcommand{\IB}{\mathsf{IB}}
\newcommand{\MUCT}{\mathsf{MUCT}}

\newcommand{\cc}{\mathsf{cc}}

\newcommand{\Bern}{\mathtt{Bern}}

\newcommand{\enablefootnotes}{
    \let\footnote\oldfootnote
}

\let\oldnl\nl
\newcommand{\nonl}{\renewcommand{\nl}{\let\nl\oldnl}}
\theoremstyle{definition} 
\newtheorem{definition}{Definition}

\theoremstyle{plain} %
\newtheorem{theorem}{Theorem}

\newtheorem{lemma}{Lemma}
\newtheorem{proposition}{Proposition}
\newtheorem{corollary}{Corollary}

\newtheorem*{lemma*}{Lemma}

\hypersetup{
     linkcolor=betterred
    ,citecolor=betterblue
    ,filecolor=betterblue
    ,urlcolor=betterblue
    ,menucolor=betterblue
    ,runcolor=betterblue
    ,linktoc=page
}
\title{\Copy{title}{On Transportability for Structural Causal Bandits}}

\author{
    \textbf{Min Woo Park \quad Sanghack Lee}\\
    Graduate School of Data Science\\
    Seoul National University\\
    Seoul, South Korea\\
    \texttt{\{alsdn0110,sanghack\}@snu.ac.kr} 
}

\begin{document}
\maketitle

\begin{abstract}
  Intelligent agents equipped with causal knowledge can optimize their action spaces to avoid unnecessary exploration. The \textit{structural causal bandit} framework provides a graphical characterization for identifying actions that are unable to maximize rewards by leveraging prior knowledge of the underlying causal structure. While such knowledge enables an agent to estimate the expected rewards of certain actions based on others in online interactions, there has been little guidance on how to transfer information inferred from arbitrary combinations of datasets collected under different conditions---observational or experimental---and from heterogeneous environments. In this paper, we investigate the structural causal bandit with \textit{transportability}, where priors from the source environments are fused to enhance learning in the deployment setting. We demonstrate that it is possible to exploit invariances across environments to consistently improve learning. The resulting bandit algorithm achieves a sub-linear regret bound with an explicit dependence on informativeness of prior data, and it may outperform standard bandit approaches that rely solely on online learning.  
\end{abstract}
\allowdisplaybreaks


\section{Introduction}\label{sec: introduction}
The multi-armed bandit (MAB) \citep{robbins1952some, lai1985asymptotically, lattimore2020bandit} problem is a pivotal topic in sequential decision-making studies, where an agent aims to maximize cumulative rewards by repeatedly choosing actions based on observed rewards, balancing the exploration-exploitation trade-off. Traditionally, MAB problems assume independence among rewards of different arms, meaning that the reward obtained from one arm provides no information about the others. This assumption limits its applicability to scenarios where dependencies between actions are common such as clinical trials, healthcare, and advertising. 

Integrating causal knowledge into a decision-making process enables an agent to model decision problems with abundant dependency structures \citep{zhang2020causal,kumor2021sequential,zhang2022online,ruan2024causal}, where structural causal models (SCMs) \citep{pearl:2k} have been employed to represent causal relationships among actions, rewards, and other relevant factors such as context and states. This approach enables agents to make informed decisions by considering how each action causally influences the reward through causal pathways \citep{bareinboim2024introduction}. 

Existing works \citep{bareinboim2015bandits,lattimore2016causal,pmlr-v70-forney17a} have shown that MAB algorithms with causal knowledge can significantly outperform others that do not account for causal dependencies. Subsequent work has explored various specialized settings by introducing additional structural assumptions, such as the availability of both observational and experimental distributions or linear mechanisms \citep{lu2020regret, bilodeau2022adaptively, feng2023combinatorial, varici2023causal}. Specifically, \citet{lee2018structural} formalized the \textit{structural causal bandit} without any parametric assumptions. Building on this, \citet{lee2019structural} extended the framework to accommodate scenarios involving non-manipulable variables. 


Although the framework significantly reduces the action space, it still requires a substantial amount of exploration, which can be costly in many applications. An alternative approach to alleviating the cost of active experimentation is to leverage previous experimental records from related environments. The expectation is that informative prior data can help narrow down reward distributions and circumvent the \textit{cold-start} problem of agents, allowing them to converge to optimal actions faster, ultimately achieving a higher cumulative reward, even without incurring any regret. However, discrepancies across environments are often significant, meaning that the data obtained from source environments may not always be informative or lead to improvements in a target environment. 

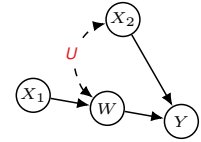
\begin{wrapfigure}[9]{r}{0.24\textwidth}
    \vspace{-0.5em}
    \centering
    \begin{tikzpicture}[x=1cm,y=.8cm,>={Latex[width=1.4mm,length=1.7mm]},
        font=\sffamily\sansmath\tiny,
        line width=0.2mm,
        RR/.style={draw,circle,inner sep=0mm, minimum size=4.5mm,font=\sffamily\tiny},
        transp/.style={regime,fill=betterred,minimum size=1mm,draw=none,inner sep=1mm}, 
        regime/.style={rectangle,draw=betterred,very thick,draw=betterred,minimum size=1mm},
        regime line/.style={draw=betterred, very thick},
        rotate = -10]\centering
        
        


    
        \node[RR] (X1) at (0,0) {$X_1$};
        \node[RR] (X2) at (1,1.5) {$X_2$};
        \node[RR] (W) at (1,0) {$W$};
        \node[RR] (Y) at (2,0) {$Y$};
        \draw[->] (X1) -- (W);
        \draw[->] (X2) -- (Y);
        \draw[->] (W) -- (Y);

        \node[] (U) at (0.4,0.8) {\color{betterred}$U$};
        \draw[->,dashed] (U) to[bend right = 20] (W);
        \draw[->,dashed] (U) to[bend left = 20] (X2);

    \end{tikzpicture}
    \caption{\small Diagram encoding causal relations.}
    \label{wrapfig: cardiovascular}
\end{wrapfigure}
\paragraph{Example. (Cardiovascular disease treatment)} Consider a scenario where $Y$ represents \textit{cardiovascular disease}, $W$ \textit{blood pressure}, $X_1$ the intake of an \textit{antihypertensive drug}, $X_2$ the use of an \textit{anti-diabetic drug}, and $U$ unobserved factors (e.g., physical activity levels, diet patterns; \citet{ferrannini2012diabetes}). Prior data from \textit{Houston} (source) is available for designing a population-level treatment strategy for cardiovascular disease patients in \textit{Boston} (target), aiming to determine appropriate medications. \Cref{wrapfig: cardiovascular} graphically illustrates this scenario. Such data can be useful but must be handled with care, especially if the population is suspected to differ, e.g., we may expect the distribution $P(U)$ vary across populations. Consider an instance where variables $X_1, X_2, W$, and $Y$ are binary, and their values are determined by functions: $W {\gets} X_1 {\oplus} U$, $X_2 {\gets} U$, and $Y {\gets} W {\land} X_2$, where $\oplus$ denotes the exclusive-or operation. $X_1$ is drawn independently and uniformly over $\{0,1\}$ and $P(U {=} 1) {=} 0.4$. Under this system, we find $\mb{E}_{P_{do(X_1 {=}1)}}Y {=} 0.6 > 0.4 {=} \mb{E}_{P_{do(X_1 {=}0)}}Y$ suggesting the antihypertensive drug was effective for Houstonians. Now suppose $P(U {=} 1) {=} 0.7$, reflecting the situation in Boston. In the target, the expected outcome becomes $\mb{E}_{P_{do(X_1 {=} 1)}}Y = 0.3$, which is the opposite of what was observed in the source. This example illustrates that the optimal strategy in a source can be suboptimal in a target. 

In the causality literature, this problem falls under the rubric of \textit{transportability theory} \citep{pearl2011transportability,bareinboim2012transportability,bareinboim2016causal}, which provides methods for determining when and how a causal effect can be computed across different environments. \citet{zhang2017transfer} studied the transfer of prior observations in settings with specific graph structures, e.g., Bow and Instrumental Variable (IV), demonstrating that leveraging existing experience can enhance the performance of an agent. \citet{bellot2023transportability} and \citet{deng2025transfer} addressed this line of work in general causal diagrams, although their focus was limited to single-node interventions. 
\paragraph{Contributions.} We propose a structural causal bandit algorithm that leverages prior data while accounting for structural discrepancies across heterogeneous environments. Our main contributions are as follows: (1) Within the structural causal bandits framework, we first establish hierarchical relations between action spaces and derive corresponding \textit{dominance} bounds on the expected rewards of actions. (2) We introduce a method for estimating the expected rewards or their \textit{causal} bounds using arbitrary combinations of observational or experimental sources. (3) We provide a UCB-based algorithm that incorporates causal information to guide exploration, and we provide both theoretical guarantees and empirical results showing it achieving a sub-linear cumulative regret depending on the amount of causal knowledge.

\section{Preliminaries}\label{sec: preliminaries}
We introduce notation and review relevant prior work. Following conventions, we use a capital letter, such as $X$, to represent a variable, with its corresponding lowercase letter, $x$, denoting a realization of the variable. Boldface is employed to represent a set of variables or values, denoted by $\mf{X}$ or $\mf{x}$. The domain of $X$ is indicated by $\Omega_X$ and $\Omega_\mf{X} {=} \times_{X \in \mf{X}}\Omega_X$. We use calligraphic letters for graphs and models such as $\mc{G}$ and $\mc{M}$. The distribution over variables $\mf{X}$ is denoted by $P(\mf{X})$. We consistently use $P(\mf{x})$ as an abbreviation for $P(\mf{X} = \mf{x})$. We denote by $\mathds{I}\{\mf{X} = \mf{x}\}$, the indicator function. 

\paragraph{Structural Causal Model.} We use \textit{structural causal model} (SCM) \citep{pearl:2k} as the semantic framework to represent the underlying environment a decision maker is deployed. An SCM $\mc{M}$ is a quadruple $\langle \mf{U}, \mf{V}, \mf{F}, P(\mf{U})\rangle$, where $\mf{U}$ is a set of exogenous variables determined by factors outside the model following a joint distribution $P(\mf{U})$, and $\mf{V}$ is a set of endogenous variables whose values are determined following a collection of functions $\mf{F} = \{f_V\}_{V \in \mf{V}}$ such that $V \gets f_V(\mf{PA}_V,\mf{U}_V)$ where $\mf{PA}_V \subseteq \mf{V} \setminus \{V\}$ and $\mf{U}_V \subseteq \mf{U}$. The observational probability $P(\mf{v})$ is defined as $\int_{\mf{u}}\prod_{V \in \mf{V}} \mathds{I}\{f_V(\mf{pa}_V,\mf{u}_V) = v\}dP(\mf{u})$. Every SCM $\mc{M}$ is associated with a \textit{causal diagram} (also called a semi-Markovian graph) $\mc{G} = \langle \mf{V}, \mf{E}\rangle$ where a directed edge $V_i \to V_j \in \mf{E}$ if $V_i \in \mf{PA}_{V_j}$, and a bidirected edge between $V_i$ and $V_j$ if $\mf{U}_{V_i}$ and $\mf{U}_{V_j}$ are dependent. The probability of $\mf{V} = \mf{v}$ when $\mf{X}$ is intervened upon to take the value $\mf{x}$ is denoted by $P(\mf{v}\mid do(\mf{x}))$ or $P_{\mf{x}}(\mf{v})$, and the submodel induced by the intervention is denoted by $\mc{M}_\mf{x}$.

\paragraph{Graphical notations.} An ordered sequence of edges is called a \textit{path}. If a path consists of directed edges with the same orientation, we say the path is \textit{directed}. A path is \textit{directed} from $X$ to $Y$ if there is no arrowhead on the path pointing towards $X$. If there is a (possibly empty) directed path from $X$ to $Y$, then $Y$ is called a \textit{descendant} of $X$, and $X$ is an \textit{ancestor} of $Y$. A variable $Y$ is referred to as a \textit{child} of $X$, and $X$ is a \textit{parent} of $Y$ if they are adjacent and the edge is not directed into $X$. We denote the ancestors, descendants, parents, and children of a given variable as $\An$, $\De$, $\Pa$, and $\Ch$, respectively. Ancestors and descendants include the variable itself. For a set of variables, we define the ancestral set as $\An(\mf{X})_{\mc{G}} = \bigcup_{X\in\mf{X}} \An(X)_{\mc{G}}$, and similarly for other relationships. The $\mf{X}$-lower-manipulation of $\mc{G}$ removes all outgoing edges from variables in $\mf{X}$, denoted as $\Go{\mf{X}}$, while the $\mf{X}$-upper-manipulation of $\mc{G}$ removes all incoming edges into variables in $\mf{X}$ in $\mc{G}$, denoted as $\Gi{\mf{X}}$. We denote the set of variables in $\mc{G}$ by $\mf{V}(\mc{G})$. A subgraph $\mc{G}[\mf{V}']$, where $\mf{V}' \subseteq \mf{V}(\mc{G})$ is defined as a vertex-induced subgraph in which all edges among the vertices in $\mf{V}'$ are preserved. We define $\mc{G} \backslash \mf{X}$ as $\mc{G}[\mf{V}(\mc{G}) \setminus \mf{X}]$ for $\mf{X} \subseteq \mf{V}(\mc{G})$. We denote $\mc{G} \langle \mf{X} \rangle$ as the latent projection of $\mc{G}$ on to $\mathbf{X}$. We provide related literature in \Cref{sec: app related works}, along with detailed background for our work in \Cref{sec: app background}.



\section{Structural Causal Bandits with Transportability}\label{sec: SCB}
We formalize the \textit{structural causal bandit with transportability} problem, where an agent interacts with a target system modeled by a structural causal model (SCM) $\mc{M}^\ast$ including a reward variable $Y \in \mf{V}$. In this setting, pulling each arm corresponds to intervening on a set of variables $\{\mf{x} \in \Omega_{\mf{X}} \mid \mf{X} \subseteq \mf{V} \setminus  \{Y\} \setminus \mf{N} \}$ where $\mf{N}$ denotes non-manipulable variables. We use the terms \textit{arm}, \textit{action}, and \textit{intervention} interchangeably, depending on the context. The agent cannot access the target system $\mc{M}^\ast$ but can observe $\mf{V}$ through online interaction by pulling an arm $do(\mf{x})$. In addition, the agent has access to data from one or more related source environments $\Pi = \{\pi^1, \pi^2,\cdots,\pi^n\}$ each associated with SCMs $\mc{M}^1, \mc{M}^2,\cdots,\mc{M}^n$. The distributions associated with $\pi^i$ under $do(\mf{x})$ will be denoted by $P^i_\mf{x}$. We use the superscript $\ast$ throughout this paper to consistently denote the target environment. 



\paragraph{Graph encoding differences among environments.} To account for environment shift, we introduce a \textit{selection diagram} \citep{bareinboim2012transportability} that captures discrepancies across environments.   

\begin{definition}[Environment discrepancy]
    Let $\pi^i$ and $\pi^j$ be environments associated with SCMs $\mc{M}_1$ and $\mc{M}_2$ conforming to a causal diagram $\mc{G}$. We denote by $\Delta^{i,j}$ a set of variables such that, for every $V \in \Delta^{i,j}$, there exists a discrepancy; either $f_V^i \neq f_V^j$ or $P^i(\mf{U}_V) \neq P^j(\mf{U}_V)$\footnote{Superscripts on $\Delta$, such as $\Delta^i$, indicate discrepancies with respect to different environments $\pi^i$, while subscripts like $\Delta_\mf{x}$ denote suboptimal gaps for the regret analysis for a bandit problem. }.  
\end{definition}


\begin{definition}[Selection diagram]
    Given a collection of discrepancies $\boldsymbol{\Delta} = \{\Delta^{\ast,i}\}_{i=1}^n$ with regard to $\mc{G} = \langle \mf{V}, \mf{E} \rangle$, let $\mf{S}^i = \{S_V \mid V \in \Delta^{\ast,i}\}$ be \textit{selection nodes}. The graph $\mc{G}^{\Delta^{\ast,i}} = \langle \mf{V}\cup \mf{S}^i, \mf{E} \cup \{S_V \to V\}_{S_V \in \mf{S}^i} \}$ is called a \textit{selection diagram}. Let $\mf{S} = \bigcup_{i=1}^n \mf{S}^i$. The \textit{collective selection diagram} $\mc{G}^{\boldsymbol{\Delta}} $ is defined as $ \langle \mf{V}\cup \mf{S}, \mf{E} \cup \{S_V \to V\}_{S_V \in \mf{S}} \}$. 
\end{definition}

We shorten $\Delta^{\ast,i}$ as $\Delta^{i}$. The collective selection diagram $\mc{G}^{\boldsymbol{\Delta}}$ encodes qualitative information about $\mc{M}^\ast$ and discrepancies $\boldsymbol{\Delta}$. The absence of a selection node pointing to a variable indicates that the causal mechanism responsible for assigning values to that variable is identical across all environments corresponding to $\pi^\ast$ and $\Pi = \{\pi^i\}_{i=1}^n$. One can view the selection nodes $\mf{S}$ as switches controlling environment shifts, and the collective selection diagram $\mc{G}^{\boldsymbol{\Delta}}$ as the causal diagram for a unified SCM representing heterogeneous SCMs following $P^{\boldsymbol{\Delta}}_\mf{x}(\mf{y} \mid \mf{w},\mf{s}^i = \mf{i}, \mf{s}^{-i} = \boldsymbol{0}) = P_\mf{x}^i(\mf{y} \mid \mf{w})$ and $P^{\boldsymbol{\Delta}}_\mf{x}(\mf{y} \mid \mf{w}, \mf{s} = \boldsymbol{0})=P^\ast_\mf{x}(\mf{y} \mid \mf{w})$ where $\mf{S}^{-i} = \mf{S} \setminus \mf{S}^i$. This representation enables probabilistic operations across the target environment $\pi^\ast$ and source environments $\Pi$.

We illustrate in \Cref{wrapfig: cardiovascular selection} a collective selection diagram $\mc{G}^{\boldsymbol{\Delta}}$ corresponding to the introductory example. In this scenario, the distributions of $W$ and $X_2$ may differ between $\mc{M}^\ast$ and $\mc{M}^1$ due to the difference in the distribution of unobserved confounder $U$, which influences both variables. This is consistent with $\boldsymbol{\Delta} = \{\Delta^1 = \{W,X_2\}\}$ and $\mf{S} = \mf{S}^1 = \{S_{W},S_{X_2}\}$. In the diagram, the selection nodes $S_W$ and $S_{X_2}$ are represented as red squares. 

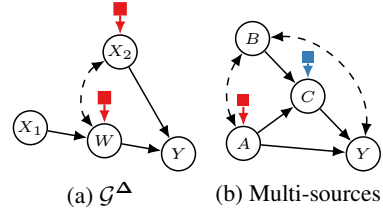
\begin{wrapfigure}[11]{r}{0.4\textwidth}
        \vspace{-1em}
        \begin{minipage}[b]{0.15\textwidth}\centering
        \subfloat[$\mc{G}^{\boldsymbol{\Delta}}$]{
        \begin{tikzpicture}[x=1cm,y=.8cm,>={Latex[width=1.4mm,length=1.7mm]},
        font=\sffamily\sansmath\tiny,
        line width=0.2mm,
        RR/.style={draw,circle,inner sep=0mm, minimum size=4.5mm,font=\sffamily\tiny},
        transp/.style={regime,fill=betterred, minimum size=1.5mm,draw=none,inner sep=0mm}, 
        regime/.style={rectangle,draw=betterred, thick,draw=betterred,minimum size=1mm},
        regime line/.style={draw=betterred, thick},
        rotate = -10]\centering

        \node[RR] (X1) at (0,0) {$X_1$};
        \node[RR] (X2) at (1,1.5) {$X_2$};
        \node[RR] (W) at (1,0) {$W$};
        \node[RR] (Y) at (2,0) {$Y$};
        \draw[->] (X1) -- (W);
        \draw[->] (X2) -- (Y);
        \draw[->] (W) -- (Y);
        \draw[<->,dashed] (W) to[bend left = 50] (X2);
        
        \node[transp, above = 2.5mm of W, draw = betterred, fill = betterred] (WT){};
        \draw[->,regime line, draw = betterred] (WT) -- (W);

        \node[transp, above = 2.5mm of X2, draw = betterred, fill = betterred] (X2T){};
        \draw[->,regime line, draw = betterred] (X2T) -- (X2);
    \end{tikzpicture}\label{wrapfig: cardiovascular selection}}
    \end{minipage}
    \hfill
    \begin{minipage}[b]{0.15\textwidth}\centering
        \subfloat[Multi-sources]{
        \begin{tikzpicture}[x=0.8cm,y=.7cm,>={Latex[width=1.4mm,length=1.7mm]},
        font=\sffamily\sansmath\tiny,
        line width=0.2mm,
        RR/.style={draw,circle,inner sep=0mm, minimum size=4.5mm,font=\sffamily\tiny},
        transp/.style={regime,fill=betterred, minimum size=1.5mm,draw=none,inner sep=0mm}, 
        regime/.style={rectangle,draw=betterred, thick,draw=betterred,minimum size=1mm},
        regime line/.style={draw=betterred,  thick},
        rotate = -5]
        
        \node[RR] (A) at (0,0) {$A$};
        \node[RR] (B) at (0,2) {$B$};
        \node[RR] (C) at (1,1) {$C$};
        \node[RR] (Y) at (2,0) {$Y$};
        \draw[->] (A) -- (Y);
        \draw[->] (A) -- (C);
        \draw[->] (B) -- (C);
        \draw[->] (C) -- (Y);


        \node[transp, above=2.5mm of A] (AT)
        {};
        \draw[->,regime line] (AT) -- (A);
        
        
        \node[transp, above=2mm of C, draw = none, fill = betterblue] (CT2)
        {};
        \draw[->,regime line, draw = betterblue] (CT2) -- (C);
        
        \draw[dashed, <->] (A) to[bend left = 30] (B);
        \draw[dashed, <->] (B) to[bend left = 50] (Y);

        \end{tikzpicture}\label{wrapfig: 4var selection}}
    \end{minipage}
    \hfill\null
    \caption{\small Collective selection diagrams for (a) the introductory example and (b) $\Delta^1 = \{A\}$ ({\color{betterred} red}) and $\Delta^2 = \{C\}$ ({\color{betterblue} blue}).}
\end{wrapfigure}
We have $P^\ast(x_1) =  P^{\boldsymbol{\Delta}}(x_1 \mid \mf{s} = \boldsymbol{0}) =P^{\boldsymbol{\Delta}}(x_1 \mid \mf{s} = \boldsymbol{1}) = P^1(x_1)$ due to the \textit{d-separation} relation \citep{pearl:95} $(\mf{S} \CI_d X_1)_{\mc{G}^{\boldsymbol{\Delta}}}$ and a similar equality can be derived for $Y$. 
In contrast, it may \textit{not} hold $P^\ast(w,x_2) = P^1(w,x_2)$ due to $(\mf{S} \not \CI_d W,X_2)_{\mc{G}^{\boldsymbol{\Delta}}}$. This result indicates that, given access to $P^1(\mf{v})$, the probability $P^\ast(x_1)$ is inferable via $P^1(x_1) = \sum_{\mf{v} \setminus \{x_1\}}P^1(\mf{v})$, whereas $P^\ast(w, x_2)$ is not. In this sense, we may say that $P^\ast(x_1)$ is transportable from the source environment $\pi^1$. We will provide, within the context of structural causal bandits, a formal definition of transportability and study it in detail in \Cref{sec: trans}. 




\begin{definition}[Structural causal bandits with transportability]
     Let $\mf{x}_t$ be the action taken at round $t \in \{ 1,\cdots, T\}$. The goal of \textit{structural causal bandits with transportability} is to minimize cumulative regret in the target environment $\pi^\ast$ defined as follows:
    \setlength{\belowdisplayskip}{0pt}%
    \setlength{\abovedisplayskip}{0pt}%
    \begin{align}
        R_T = \sum_{t=1}^T \mb{E}_{P^\ast_{\mf{x}^{\star}}} Y - \mb{E}_{P^\ast_{\mf{x}_t}} Y
        = \sum_{\mf{x}}\Delta_\mf{x} \mb{E}N_T(\mf{x}) \label{eq: def regret}
    \end{align}
    that compares the reward of the optimal arm $\mf{x}^{\star} =
        \arg\max_{ \mf{x} \in \Omega_{\mf{X}}, \mf{X} \subseteq \mf{V} \setminus \{Y\} \setminus \mf{N}}\mb{E}_{P^\ast_{\mf{x}}}Y$ with that of arm $\mf{x}_t$ in each round $t$. $\Delta_\mf{x}$ denotes suboptimal gap $\mb{E}_{P^\ast_{\mf{x}^{\star}}} Y - \mb{E}_{P^\ast_{\mf{x}}} Y$ and $N_T(\mf{x})$ denotes the number of times an action $\mf{x}$ was chosen up to round $T$.  
\end{definition}

\paragraph{Action space worth exploring.} In settings where variables exhibit causal relationships represented by a causal diagram, restricting attention to a subset of the entire action space can lead to improved performance. This implies that instead of exploring all the exponential subsets in $2^{\mf{V} \setminus \{Y\} \setminus \mf{N}}$, it suffices to consider some subspace. We denote by $\mf{x}^\ast = \arg \max_{\mf{x} \in \Omega_{\mf{X}}} \mb{E}_{P_{\mf{x}}} Y$ the best expected reward by intervening on $\mf{X}$, and $\mf{x}[\mf{X}']$ the values of $\mf{x}$ restricted to the subset of variables of $\mf{X} \cap \mf{X}'$. 

\begin{definition} [Minimality \citep{lee2018structural}] \label{def: mis}
    If $\mf{X} \subseteq \mf{V} \setminus \{Y\} \setminus \mf{N}$ be a set such that there is no $\mf{X}' \subseteq \mf{X}$ such that $\mb{E}_{P_{\mf{x}}}Y = \mb{E}_{P_{\mf{x}[\mf{X}']}}Y$\footnote{We refer to $\mf{X}$ and $\mf{X}'$ as \textit{equivalent} if the equality holds.}, we refer to $\mf{X}$ as a \textit{minimal} intervention set (MIS). 
\end{definition}

\begin{definition} [Possibly-optimal intervention set] \label{def: POIS}
    Let $\mf{X} \subseteq \mf{V} \setminus \{Y\} \setminus \mf{N}$ be a set of variables. We say that $\mf{X}$ is a \textit{possibly-optimal intervention set} (POIS) with respect to $\langle \mc{G}, Y, \mf{N} \rangle$ if there exists an SCM conforming to $\mc{G}$ such that $\mb{E}_{P_{\mf{x}^{\ast}}}Y  >\mb{E}_{P_{\mf{w}^\ast}}Y$ for all $\mf{W}\subseteq \mf{V} \setminus \{Y\} \setminus \mf{N}$ that are \textit{not} equivalent to $\mf{X}$. 
\end{definition}


Minimality implies that every variable $X\in \mf{X}$ affects the reward variable without passing through $\mf{X}\setminus \{X\}$ in $\mathcal{G}$.
We refer to a set as a \textit{possibly-optimal minimal intervention set} (POMIS) \citep{lee2018structural,lee2019structural} if it is both a POIS and minimal. We denote by $\mb{P}_{\mc{G},Y}^\mf{N}$ a set of POMISs with respect to $\langle \mc{G},Y,\mf{N} \rangle$\footnote{For readability, we omit $\mf{N}$ when $\mf{N} = \emptyset$ (e.g., $\mb{P}_{\mc{G},Y}$ and $\langle \mc{G},Y\rangle$), referring to this case as \textit{unconstrained}.}. By definition of POMIS, intervening on non-POMIS cannot yield a better outcome than the optimal one associated with POMIS. This means $\mf{x}^{\star}$ in \Cref{eq: def regret} can be equivalently expressed as 
$\mf{x}^{\star} =
        \arg\max_{ \mf{x} \in \Omega_{\mf{X}}, \mf{X} \in \mb{P}_{\mc{G},Y}^{\mf{N}^\ast}}\mb{E}_{P^\ast_{\mf{x}}}Y$ restricting the exploration space to POMISs. Therefore, an agent who is aware of POMIS should only explore and exploit actions consistent with those sets. All graphical characterizations for PO(M)IS are in \Cref{sec: app POIS}. 
\subsection{Dominance Relationships among Action Spaces}
We say an action space \textit{dominates} another when it behaves better than or equal to another with respect to maximum achievable expected reward. For example, it is immediate from the definition of POIS (\Cref{def: POIS}) that all POISs with respect to $\langle \mc{G}, Y, \mf{N} \rangle$ dominate any non-POISs under the same constraint. Let $\mf{W} \subseteq \mf{V} \setminus \{Y\} \setminus \mf{N}^\ast$ be a set that is \textit{not} a POIS with respect to $\langle \mc{G}, Y, \mf{N}^\ast \rangle$.
According to \Cref{def: POIS}, we can derive that $\mf{W}$ cannot outperform the target POMISs with respect to $\langle \mc{G},Y,\mf{N}^\ast \rangle$ (i.e., POMISs in the target environment). The following inequality states that \textit{the target POMIS dominates sets which are non-POISs}. 

\setlength{\belowdisplayskip}{3pt}%
\setlength{\abovedisplayskip}{5pt}%
\begin{align}
\mathbb{E}_{P^\ast_{\mf{x}^{\star}}} Y
\geq 
\mathbb{E}_{P^\ast_{\mf{w}^{\ast}}} Y \label{eq: dominace lower}
\end{align}
 
This inequality implies that if there exists at least one non-POIS action $\mf{w} \in \Omega_{\mf{W}}$ whose expected reward is greater than that of any target POMIS action, then such an action cannot be the true optimal action. Beyond such trivial cases, we now turn our attention to more nuanced dominance relations that arise between constrained and unconstrained PO(M)ISs. Let $\mf{r}^\star$ be an optimal POIS action with respect to $\langle \mc{G}, Y \rangle$. The following inequality says \textit{unconstrained POISs dominates target POMISs}. 
 
    \setlength{\belowdisplayskip}{0pt}%
    \setlength{\abovedisplayskip}{3pt}%
    \begin{align}
    \mathbb{E}_{P^\ast_{\mf{x}^{\star}}} Y
    \leq
    \mathbb{E}_{P^\ast_{\mf{r}^{\star}}} Y. \label{eq: dominace upper} 
    \end{align}

This implies that the expected rewards of target POMIS arms are upper bounded by the right-hand side of \Cref{eq: dominace upper}. To witness, consider the ongoing cardiovascular example (\Cref{wrapfig: cardiovascular selection}). Suppose blood pressure $W$ is non-manipulable, i.e., $\mf{N}^\ast = \{W\}$. The set of POMISs is then given by $\mb{P}_{\mc{G},Y}^{\mf{N}^\ast} = \{\{X_1\},\{X_1,X_2\} \}$, which implies that the optimal action $\mf{x}^\star$ must be consistent with $do(x_1^\ast)$ or $do(\{x_1,x_2\}^\ast)$ but not with $do(\emptyset)$ or $do(x_2^\ast)$. According to the dominance relationship, $do(x_1)$ and $do(x_1,x_2)$ can be interpreted as the \textit{best alternative} plans under the constraint $\mf{N}^\ast= \{W\}$. For concreteness, we consider an unconstrained (i.e., $\mf{N} = \emptyset$) POIS $\mf{R} = \{W,X_2\}$. The expected reward under $do(x_1^\ast)$ can be decomposed as $\mb{E}_{P^\ast_{x_1^\ast}} Y = \sum_{\mf{r}} \mb{E}_{P^\ast_{x_1^\ast}}[Y\mid \mf{r}] P^\ast_{x_1^\ast}(\mf{r}) 
\overset{\text{(a)}}{=} 
\sum_{\mf{r}} \mb{E}_{P^\ast_{x_1^\ast,\mf{r}}} [Y] P^\ast_{x_1^\ast}(\mf{r}) 
\overset{\text{(b)}}{=} 
\sum_{\mf{r}} \mb{E}_{P^\ast_{\mf{r}}} [Y] P^\ast_{x_1^\ast}(\mf{r}) 
\leq
\mb{E}_{P^\ast_{\mf{r}^\ast}} Y = \mb{E}_{P^\ast_{\mf{r}^\star}} Y$ where (a) follows from Rule~2 and (b) follows from Rule~3 of do-calculus. This inequality shows that $\mb{E}_{P^\ast_{x_1^\ast}} Y \leq \mb{E}_{P^\ast_{\mf{r}^\star}} Y$, and a similar argument applies to $do(\{x_1,x_2\}^\ast)$ implying that $\mb{E}_{P^\ast_{\mf{x}^\star}} Y \leq \mb{E}_{P^\ast_{\mf{r}^\star}} Y$. In fact, one can observe $\mb{E}_{P_{\{x_1,x_2\}^\ast}^\ast} = \mb{E}_{P_{x_1^\ast}^\ast} Y = 0.7 \leq \mb{E}_{P_{\mf{r}^\star}^\ast} = 1$. We now turn to the dominance relations involving non-POIS actions in \Cref{eq: dominace lower}. Consider $\mf{W} = \{X_2\}$ a non-POIS with respect to $\langle \mc{G}, Y, \mf{N}^\ast \rangle$. We find $\mb{E}_{P^\ast_{\mf{w}^\ast}} Y = 0.5$, which implies that the optimal action $\mf{x}^\star$ must satisfy $\mb{E}_{P^\ast_{\mf{x}^\star}} \geq 0.5$. Indeed, as previously shown, $\mb{E} _{P^\ast_{\mf{x}^\star}} = 0.7$. Based on the dominance relations in \Cref{eq: dominace upper,eq: dominace lower}, we introduce the following dominance relationship.  
\begin{restatable}[Dominance relationship]{theorem}{thmDominanceRelations}\label{thm: dominance relations}
Let $\mf{r}^\star$ be an optimal action with respect to $\langle \mc{G}, Y, \mf{N}\rangle$ where $\mf{N}$ is a subset of $\mf{N}^\ast$. Let $\mf{W}$ be a non-POIS with respect to $\langle \mc{G}, Y, \mf{N}^\ast \rangle$. Then $\mathbb{E}_{P^\ast_{\mf{x}^\star}} Y$ is bounded by 
\setlength{\belowdisplayskip}{3pt}%
\setlength{\abovedisplayskip}{5pt}%
\begin{align}
\mathbb{E}_{P^\ast_{\mf{w}^{\ast}}} Y \leq 
\mathbb{E}_{P^\ast_{\mf{x}^{\star}}} Y
\leq
\mathbb{E}_{P^\ast_{\mf{r}^\star}} Y .\label{eq: dominace relations} 
\end{align}
\end{restatable}
%
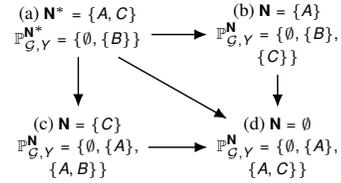
\begin{wrapfigure}[12]{r}{0.35\textwidth}


    \vspace{-1.3em}
    \begin{minipage}[b]{0.35\textwidth}\centering
        \begin{tikzpicture}[x=2.65cm,y=1.5cm,>={Latex[width=1.4mm,length=1.7mm]},
        font=\sffamily\sansmath\tiny,
        line width=0.19mm,
       every node/.style={scale=1},
        RR/.style={draw,circle,inner sep=0mm, minimum size=4.5mm,font=\sffamily\tiny}]\centering
        \node[align=center] (C) at (0,0) {
          {\rm \scriptsize (c)} $\mf{N} = \{C\}$\\
          $\mb{P}_{\mc{G},Y}^{\mf{N}} = \{\emptyset,\{A\},$\\
          $ \{A,B\}\}$
        };
        
        \node[align = center] (E) at (1,0) {
          {\rm \scriptsize (d)} $\mf{N} = \emptyset$\\
          $\mb{P}_{\mc{G},Y}^{\mf{N}} = \{\emptyset,\{A\},$\\
          $ \{A,C\}\}$
        };
        
        \node[align = center] (A) at (1,1) {
          {\rm \scriptsize (b)} $\mf{N} = \{A\}$\\
          $\mb{P}_{\mc{G},Y}^{\mf{N}} = \{\emptyset,\{B\},$\\
          $\{C\}\}$
        };
        
        \node[align = center] (AC) at (0,1) {
          {\rm \scriptsize (a)} $\mf{N}^\ast = \{A,C\}$\\
          $\mb{P}_{\mc{G},Y}^{\mf{N}^\ast} = \{\emptyset,\{B\}\}$\\\,
        };
        
        \draw[->] (AC) -- (A);
        \draw[->,shorten <=-2mm] (AC) -- (C);
        
        \draw[->,shorten <=-4mm,shorten >=-3mm] (AC) -- (E);

        \draw[->] (A) -- (E);
        \draw[->] (C) -- (E);

        \end{tikzpicture}
    \end{minipage}
    \caption{Hierarchical relationships between POMISs under different constraints. Arrows indicate the direction of dominance relations. }
    \label{wrapfig: hierarchy}
\end{wrapfigure}
The target POMISs dominate non-POISs under the same constraint, while being dominated by PO(M)ISs defined under a \textit{weaker constraint}. The second inequality in \Cref{eq: dominace relations} can be interpreted as a generalized version of \Cref{eq: dominace upper} where the target POMISs is regarded as POMISs with respect to $\langle \mc{G} \langle \mf{V} \setminus \mf{N} \rangle, Y, \mf{N}^\ast \setminus \mf{N} \rangle$. Therefore, each $\mb{E}_{P_{\mf{r}^\star}^\ast} Y$ obtained under a weaker constraint (any subset of $\mf{N}^\ast$) provides a valid upper bound for the target POMISs. While one might be concerned that evaluating POMISs for \textit{all} weaker constraints is computationally exhaustive, it is worth noting that any proper subset of $\mf{N}$ corresponds to a strictly weaker constraint than $\mf{N}$, and the associated bound is \textit{looser}; and thus the proper subsets of $\mf{N}$ dominate $\mf{N}$. We illustrate this in \Cref{wrapfig: hierarchy}, where the target POMISs (a) are defined under $\mf{N}^\ast = \{A,C\}$ for the causal diagram corresponding to \Cref{wrapfig: 4var selection}. In this setting, each optimal action under a weaker constraint, (b--d), dominates (a); furthermore, (d) dominates both (b) and (c). This means that (d) is not tighter than either (b) or (c); thus if (b) or (c) is known, estimating (d) is unnecessary to upper bound (a). 

Equipping with this dominance knowledge, the corresponding algorithm (Alg.~\ref{alg: UDB} in \Cref{sec: app dominance bound}) hierarchically traverses the space of constrained POMISs and stops \textit{early} when a transferable $\mathbb{E}_{P^\ast_{\mf{r}^\star}} Y$ is found, thereby avoiding unnecessary computation. However, it is possible that while $\mathbb{E}_{P^\ast_{\mf{r}^\star}} Y$ is \textit{not} transferable from the sources $\Pi$, its upper bound $u_{\mf{r}^\star}$ can still be transferred; thus, it provides valid bounds when integrated with \Cref{thm: dominance relations}, $\mathbb{E}_{P^\ast_{\mf{x}^{\star}}} Y \leq u_{\mf{r}^\star}$ (we will refer to this as \textit{dominance bounds}). In the following section, we describe how to compute such valid bounds of expected rewards from the sources $\Pi$ (\Cref{subsec: trans1}) and how to leverage the dominance bounds in online interaction (\Cref{subsec: trans3}).

\section{Transporting Actions from Source Environments}\label{sec: trans}

In this section, we investigate transportability of expected rewards in detail and present a method leveraging this knowledge. Consider the collective selection diagram $\mc{G}^{\boldsymbol{\Delta}}$ in \Cref{wrapfig: 4var selection}, and suppose that the first dataset from $\pi^1$ is collected under observation $\mb{Z}^1 = \{\emptyset\}$ while dataset from $\pi^2$ is obtained under \textit{randomized controlled trial} (RCT) with respect to $\mb{Z}^2 = \{\{A,C\}\}$. The target environment $\pi^\ast$ is assumed to be subject to the constraint $\mf{N}^\ast = \{C\}$. Our objective is to identify transportable quantities to minimize unnecessary exploration by leveraging prior information from $\pi^1$ and $\pi^2$. As a representative instance, the causal effect of the POMIS action $do(a)$ in $\pi^\ast$ can be written as: 
\setlength{\belowdisplayskip}{-2pt}%
\setlength{\abovedisplayskip}{-2pt}%
    \setlength{\jot}{-4pt}
\begin{align}
    P^\ast_a(y) &= \sum_{b, c}P^\ast_a(b)P^\ast_a(c \mid b)P^\ast_a(y \mid b,c) \overset{\text{(a)}}{=} \sum_{b, c} P^\ast_{a,c}(b) P^\ast_\emptyset(c \mid a,b) P^\ast_{a,c}(y \mid b)\nonumber\\
    &=\sum_{b, c} P_\emptyset^\ast(c \mid a,b) P^\ast_{a,c}(b,y)
    \overset{\text{(b)}}{=}\sum_{b, c} P_\emptyset^1(c \mid a,b) P^\ast_{a, c}(b,y)
    \overset{\text{(c)}}{=}\sum_{b, c} P_\emptyset^1(c \mid a,b) P^2_{a, c}(b,y).\label{eq: transport example}
\end{align}
Specifically, note that equality (a) follows from applications of Rule 2 (converting conditioning on $c$ to $do(c)$ and $do(a)$ to conditioning on $a$) and Rule 3 (adding $do(c)$) of do-calculus. Equation (b) holds since the discrepancy between $\pi^\ast$ and $\pi^1$ is irrelevant due to $(S_A \CI_d C \mid A,B)_{\mc{G}^{\boldsymbol{\Delta}}}$. The final equality (c) is derived from the indifference to the disparities in mechanisms $f_B$ and $f_Y$ between $\pi^\ast$ and $\pi^2$. We observe that the causal effect $P^\ast_a(y)$ can be expressed in terms of $P^1_\emptyset$ and $P^2_{a, c}$, indicating that $\mb{E}_{P_a^\ast} Y {=} \sum_{y}P_{a}^\ast(y)$ can be estimated from the sources $\Pi$. Accordingly, we say $P^\ast_a(y)$ (or equivalently $\mb{E}_{P^\ast_a} Y$) is \textit{transportable}. We now provide a formal definition of transportability. 


\begin{definition}[Transportability \citep{lee2020general}] Let $\mc{G}^{\boldsymbol{\Delta}}$ be a collective selection diagram with respect to $\Pi = \{\pi_1, \cdots, \pi_n\}$ with a target domain $\pi^\ast$. Let $\mb{Z} = \{\mb{Z}^i\}_{i=1}^n$ be a specification of available prior $\mb{Z}^i$ conducted in source environment $\pi^i$. We say that $P^\ast_\mf{x}(y)$ is \textit{transportable} with respect to $\langle \mc{G}^{\boldsymbol{\Delta}}, \mb{Z} \rangle$ if $P^\ast_\mf{x}(y)$ is uniquely computable from $\mc{P}^{\Pi}_\mb{Z} = \{P_\mf{z}^i \mid \mf{z} \in \Omega_{\mf{Z}
}, \mf{Z} \in \mb{Z}^i \in \mb{Z} \}$ in any collection of models that induce $\mc{G}^{\boldsymbol{\Delta}}$. 
\end{definition}




We introduce the following graphical concepts, which are widely used in the causality literature. Let $\cc_\mc{G} = \{\mf{C}_q\}_{q=1}^m$ be the collection of \textit{c-components} of $\mc{G}$. For $\mf{C} \subseteq \mf{V}$, we define the quantity $Q[\mf{C}](\mf{v}) = P_{\mf{v}\setminus \mf{c}}(\mf{c})$, which corresponds to the post-intervention distribution and is referred to as a \textit{c-factor}. For convenience, we omit input $\mf{v}$ and write $Q[\mf{C}]$. We denote the quantities for the target environment $\pi^\ast$ as $Q^\ast[\mf{C}] =  P_{\mf{v} \setminus \mf{c}}^{\boldsymbol{\Delta}}(\mf{c} \mid \mf{s} = \boldsymbol{0})$ and for sources $\pi^i$ as $Q^i[\mf{C}] = P_{\mf{v} \setminus \mf{c}}^{\boldsymbol{\Delta}}(\mf{c}  \mid \mf{s}^i = \mf{i}, \mf{s}^{-i} = \boldsymbol{0})$ consistently. We denote by $\mf{Y}^{\boldsymbol{+}} = \An(Y)_{\mc{G}_{\underline{\mf{X}}}}$ the set of variables that affect the reward $Y$ under the intervention on $\mf{X}$. 
The expected reward in the target environment $\mb{E}_{P^\ast_{\mf{x}}} Y$ can be uniquely expressed using $m$ c-factors ${\mf{C}_q \in \cc_{\mc{G}[\mf{Y}^{\boldsymbol{+}}]}}$ as follows:  
\setlength{\belowdisplayskip}{-1pt}%
\setlength{\abovedisplayskip}{-4pt}%
\begin{align}
    \mb{E}_{P^\ast_{\mf{x}}} Y = 
    \sum_{y} y P^\ast_{\mf{x}}(y) 
    = \sum_{\mf{y}^{\boldsymbol{+}} } y \prod_{q=1}^mQ^\ast[\mf{C}_q]. \label{eq: decompose reward}
\end{align}    
When a c-component $\mf{C}$ satisfies $\mf{C} \cap \Delta^i = \emptyset$ and there exists a c-component $\mf{C} \subseteq \mf{C}'$ such that $Q^i[\mf{C}]$ is \textit{identifiable}\footnote{This means that the quantity $Q^i[\mf{C}] = P_{\mf{v} \setminus \mf{c}}^{\boldsymbol{\Delta}}(\mf{c}  \mid \mf{s}^i = \mf{i}, \mf{s}^{-i} = \boldsymbol{0})$ is uniquely computable from $\mc{G}[\mf{C}']$. We provide the formal definition of identifiability in \Cref{def: identification}, along with further detailed description in \Cref{sec: app background}. } from $\mc{G}[\mf{C}']$, we refer to $Q^\ast[\mf{C}]$ as being \textit{transportable} from $\pi^i$ \citep{lee2020general}. 
Furthermore, ${P^\ast_\mf{x}} (Y)$ and $\mb{E}_{P^\ast_\mf{x}} Y$ are transportable with respect to $\langle \mc{G}^{\boldsymbol{\Delta}}, \mb{Z} \rangle$ \textit{if and only if} all c-factors $Q^\ast[\mf{C}_q]$ in the right-hand side of \Cref{eq: decompose reward} are transportable from some source environment $\pi^i \in \Pi$. 
%
%
To illustrate, we now reformulate \Cref{eq: transport example} in terms of c-factors, following \Cref{eq: decompose reward}: $\mb{E}_{P^\ast_a} Y = \sum_{y} y P^\ast_a(y) = \sum_{b,c,y} y Q^\ast[C]Q^\ast[B,Y]$. Since $\{C\} \cap \Delta_1 = \emptyset$ and $\{B, Y\} \cap \Delta_2 = \emptyset$, and each corresponding c-factor is identifiable from $\pi^1$ and $\pi^2$, respectively, each c-factor is transportable from $\pi^1$ and $\pi^2$. Therefore, $\mb{E}_{P^\ast_a} Y$ is transportable and is given by $\mb{E}_{P^\ast_a} Y = \sum_{b,c,y} y Q^1[C]Q^2[B,Y]$.  

When $\mb{E}_{P^\ast_\mf{x}} Y$ is not transportable, it may seem that a learning agent cannot obtain any assistance from sources and must estimate outcomes entirely from scratch, engaging in \textit{cold} exploration. However, while determining the exact values may be infeasible, the learner may still extrapolate partial knowledge from the prior to improve estimates within a feasible interval. 

\subsection{Bounding Non-transportable Actions}\label{subsec: trans1}



Our next result concerns the derivation of bounds for the target c-factor $Q^\ast[\mf{C}]$ from a non-identifiable source quantity $Q^i[\mf{C}]$ in terms of an identifiable source quantity $Q^i[\mf{C}']$ where $\mf{C} \subseteq \mf{C}'$. 

\begin{restatable}[]{proposition}{propQTransBound}\label{prop: Q trans bound}
    Let $\mf{C}$ be a c-component in $\mc{G} \setminus \mf{X}$ satisfying $\mf{C} \cap \Delta^i = \emptyset$ and $\mf{D}$ be a c-component satisfying $\mf{C} \subseteq \mf{D}$. Let $\mf{C}' =  \An(\mf{C})_{\mc{G}[\mf{D}]}$. The target c-factor $Q^\ast[\mf{C}]$ is bounded in $[\ell,
        u]$ where (i) if $\mf{C} = \mf{C}'$, then $\ell = u = \sum_{\mf{c}' \setminus \mf{c}} Q^i[\mf{C}']$; (ii) otherwise, $\ell = Q^i[\mf{C'}] $
            and $u = Q^i[\mf{C}'] +1 - \sum_{\mf{c}}Q^i[\mf{C}']$. 
\end{restatable}

We now revisit the collective selection diagram $\mc{G}^{\boldsymbol{\Delta}}$ in \Cref{wrapfig: 4var selection} under a more challenging setting, where the source is modified by supposing that the given distributions correspond to $\mb{Z}^1 = \{\emptyset, \{C\}\}$ and $\mb{Z}^2 = \{\{B\}\}$. From $\pi^1$, although $\{B, Y\} \cap \Delta_1 = \emptyset$, $Q^1[B,Y]$ is non-identifiable; consequently, $Q^\ast[B,Y]$ is non-transportable from $\pi^1$. In $\pi^2$, there exists no c-component involving $B$, rendering the target c-factor non-transportable. This structural limitation prevents the learner from identifying $Q^\ast[B, Y]$, failing transportability of $\mb{E}_{P^\ast_a} Y$. While $Q^\ast[B,Y]$ is not transportable from either $\pi^1$ or $\pi^2$, it is important to note that $Q^\ast[B,Y] = Q^1[B,Y]$ holds due to $\{B,Y\} \cap \Delta_1 = \emptyset$. This implies that any valid bound on $\ell \leq Q^1[B, Y] \leq u$ also induces a valid bound on $ Q^\ast[B, Y]$. Following \Cref{prop: Q trans bound}, we consider $\mf{C}' = \{A, B, Y\}$; the lower bound $\ell$ is given by $Q^1[A, B, Y]$, and the upper bound $u$ is $Q^1[A, B, Y] + 1 - \sum_{b,y} Q^1[A, B, Y]$. Therefore, bound on $\mb{E}_{P^\ast_a} Y$ is written as 
\setlength{\belowdisplayskip}{-2pt}%
\setlength{\abovedisplayskip}{-2pt}%
\[{\sum_{b,c,y} y P^1_{\emptyset}(c \mid a,b)P^1_c(a,b,y) \leq\mb{E}_{P^\ast_a} Y \leq \sum_{b,c,y} y P^1_{\emptyset}(c \mid a,b)\left\{P^1_c(a,b,y) + 1 - P^1_c(a)\right\}}\] 
which is derived by $Q^1[C] = P^1(c\mid a,b)$ and $Q^1[A,B,Y] = P^1_c(a,b,y)$. With this in hand, we are ready to formally construct valid bounds of expected rewards of target actions given $\langle \mc{G}^{\boldsymbol{\Delta}}, \mb{Z} \rangle$. 
%
\begin{restatable}[Causal bounds]{theorem}{thmCausalBound}\label{thm: causal bound}
    Given $\langle \mc{G}^{\boldsymbol{\Delta}}, \mb{Z} \rangle$, the target expected reward $\mb{E}_{P^\ast_\mf{x}} Y$ can be bounded by $[\ell_\mf{x},u_\mf{x}]$ if for all c-factors $Q^\ast[\mf{C}_q]$ in the right-hand side of \Cref{eq: decompose reward}, there exists a source $\pi^i \in \Pi$  satisfying $\mf{C}_q \cap \Delta^i = \emptyset$ and a computable ancestral c-component $\mf{C}_q \subseteq \mf{C}_q'$ from $\mc{G} \setminus \mf{Z}$ where $\mf{Z} \in \mb{Z}^i \in \mb{Z}$. The bound $[\ell_\mf{x},u_\mf{x}]$ is defined as: 
    \setlength{\belowdisplayskip}{0pt}%
    \setlength{\abovedisplayskip}{0pt}%
    \setlength{\jot}{-4pt}
    \begin{align}
    \ell_{\mf{x}} \triangleq 
    \sum_{\mf{y}^{\boldsymbol{+}}} y&
        \textcolor{NavyBlue}{\prod_{q=1}^k
        Q^{i_q}[\mf{C}_q] }
        {
        \prod_{q = k+1}^m
        Q^{j_{q}}[\mf{C}'_{q}]} 
        \leq 
        \sum_{\mf{y}^{\boldsymbol{+}}} y
        \textcolor{NavyBlue}{\prod_{q=1}^k
        Q^{i_q}[\mf{C}_q] }
        {
        \prod_{q = k+1}^m
        Q^{\ast}[\mf{C}_{q}]}\label{eq: casual lower bound}\\
    &\leq 
     \sum_y y \min\Bigl\{1, \sum_{\mf{y}^{\boldsymbol{+}} \setminus \{y\}}
        \textcolor{NavyBlue}{\prod_{q=1}^k
        Q^{i_q}[\mf{C}_q] }
        {
        \prod_{q = k+1}^m
        \big\{ Q^{j_{q}}[\mf{C}'_{q}]  +1 - \sum_{\mf{c}_{q}}Q^{j_{q}}[\mf{C}'_{q}] \big \}} \Bigr\}
     \triangleq u_{\mf{x}} \label{eq: casual upper bound}
    \end{align}
    where $\textcolor{NavyBlue}{\prod_{q=1}^{k}Q^{i_q}[\mf{C}_q]}$ denotes transportable c-factors.
\end{restatable}
The upper bounds of $P_{\mf{x}}^\ast(y)$ may exceed one due to sum-product operations over non-transportable terms. Thus, we take the minimum of the value and one. The corresponding algorithms (Algs.~\ref{alg: paTR} and ~\ref{alg: causal bound}) are presented in \Cref{sec: app partial trans}. First, the algorithm \textsc{paTR} (Alg.~\ref{alg: paTR}) outputs \textit{all} expressions for bounds of $P_{\mf{x}}^\ast(y)$, based on \Cref{prop: Q trans bound} and \Cref{thm: causal bound}. Then, \textsc{CausalBound} (Alg.~\ref{alg: causal bound}) computes the bounds on $\mb{E}_{P^\ast_\mf{x}} Y$ over the sources $\Pi$, and returns the tightest lower and upper bounds as the causal bound.

\subsection{Upper Confidence Bound Algorithm with Transport Bounds}\label{subsec: trans3}

In this section, we now incorporate obtained bounds into a bandit problem. Dominance relationships (\Cref{thm: dominance relations}) and causal bounds (\Cref{thm: causal bound}) can refine the upper confidence bound (UCB) \citep{auer2002finite} estimates during online learning. To illustrate this, we follow the \textit{clipped upper confidence bound} approach \citep{zhang2017transfer,Zhang2023towards} as our index strategy. We denote the empirical reward estimate at round $t$ for each arm in the target environment $\pi^\ast$ as $ \hat{\mb{E}}_{P_{\mf{x}}^\ast,t} Y = \frac{1}{N_t(\mf{x})} \sum_{t'=1}^t Y_{\mf{x},t'}  \mathds{I}\{\mf{X}_{t'} = \mf{x}\}$. The standard UCB index is defined as $U_\mf{x}(t) = \hat{\mb{E}}_{P^\ast_{\mf{x}},t}Y + \sqrt{\frac{\ln(1/\delta)}{2N_t(\mf{x})}}$ where $\delta = t^{-4}$. Our index policy is defined as $\bar{U}_{\mf{x}}(t) = \min\{ \max\{ U_{\mf{x}}(t), \ell_\mf{x}\} ,u_\mf{x}\}$ which constrains the standard UCB index within the final transport bounds $[\ell_\mf{x},u_\mf{x}]$. Our algorithm \textsc{trUCB} is presented in Alg.~\ref{alg: trUCB}, which begins by initializing the target action space with POMISs: $\mc{I}^\ast \triangleq \{\mf{x} \in \Omega_{\mf{X}} \mid \mf{X} \in \mb{P}_{\mc{G},Y}^{\mf{N}^\ast}\}$ (Line 1)\footnote{Each subroutine, $\MISs$ and $\POMISs$, is an algorithm that returns MISs and POMISs, respectively \citep{lee2018structural}. Given $\mc{G} \langle \mf{V} \setminus \mf{N}^\ast \rangle$ and $Y$, these algorithms compute the corresponding sets for $\langle \mc{G}, Y, \mf{N}^\ast \rangle$. }. 



\paragraph{Dominance bounds and causal bounds.} The next part (Lines 2--4), the algorithm attempts to determine whether the causal bounds can be identified and compute them using \Cref{thm: causal bound} (corresponding to \textsc{CausalBound}). All causal bounds are initialized as $[0,\infty)$. Using the resulting causal bounds for each action and dominance relationship (\Cref{thm: dominance relations}), the algorithm computes the \textit{upper dominance bound} $u^\star$ by minimizing expected rewards of POMIS actions (or upper causal bounds if non-transportable) for weaker constraints, which may result in a tighter bound. The next step is to obtain the \textit{lower dominance bound} $\ell^{\star}$ by taking the maximum of causal lower bounds $\ell_{\mf{w}}$ over minimal actions under the same constraint (i.e., MISs with respect to $\langle \mc{G},Y,\mf{N}^\ast \rangle$). Since such actions can never outperform the optimal arm $\mf{x}^\star$, we can safely exclude any actions $\mf{x} \in \mc{I}^\ast$ whose upper causal bound $u_{\mf{x}}$ is lower than the lower dominance bound $\ell^\star$.  We defer the technical details of the subroutines---\textsc{CausalBound} and \textsc{udb}---to \Cref{sec: app partial trans,sec: app dominance bound}, respectively. After completion of this phase, we refer to $[\ell_{\mf{x}}, u_{\mf{x}}]$ as \textit{transport bounds}, as they incorporate both dominance and causal bounds. 


\paragraph{Clipped UCB.} In the last part (Lines 5--7), the algorithm enters the online interaction phase with the final actions space $\mc{I}^\ast$. At each round $t$, it computes the clipped UCB $\bar{U}_{\mf{x}}(t)$ for every arm by combining the empirical rewards collected up to round $t$ with the transport bounds $[\ell_{\mf{x}},u_{\mf{x}}]$. The agent then selects an arm with the highest index and receives the corresponding reward. Not surprisingly, this strategy ensures that the cumulative regret grows sublinearly with the number of rounds $T$. 

\begin{algorithm}[t]
\small
\SetAlgoNoEnd
\DontPrintSemicolon
\SetKwInput{KwInput}{Input}

\KwInput{$Y$: reward; $\mc{G}$: causal diagram; $\mf{N}^\ast$: non-manipulable variables; $\boldsymbol{\Delta}$: discrepancies; $\mb{Z}$: a specification of priors;
$\Pi$: sources;
$\mc{P}_{\mb{Z}}^{\Pi}$: available distributions;
}

\setcounter{AlgoLine}{0}

Initialize the target action space $\mc{I}^\ast \gets \{\mf{x} \in \Omega_{\mf{X}} \mid \mf{X} \in \textsf{POMISs}(\mc{G} \langle \mf{V} \setminus \mf{N}^\ast \rangle ,Y)\}$ 




Set the causal bounds $[\ell_{\mf{x}}, u_{\mf{x}}]$ for all actions using \textsc{CausalBound} (Alg.~\ref{alg: causal bound} in \Cref{sec: app partial trans})


Set the upper dominance bound $u^\star$ using \textsc{udb} (Alg.~\ref{alg: UDB} in App.~\ref{sec: app dominance bound}); and $u_\mf{x} \gets \min\{u_\mf{x}, u^\star\}$ for all $\mf{x} \in \mc{I}^\ast$.  


Compute $\ell^{\star} = \max_{\mf{w} \in \Omega_{\mf{W}}, \mf{W} \in \MISs(\mc{G}\langle \mf{V} \setminus \mf{N}^\ast \rangle , Y)} \ell_{\mf{w}}$; and remove actions from $\mc{I}^\ast$ such that $u_\mf{x} < \ell^\star$. 


\For{\rm each trial $t \leq T$}{
    Choose an arm $\mf{x}_t =
        {\mathrm{argmax~}}_{\mf{x} \in \mc{I}^\ast}
        \bar{U}_{\mf{x}}(t)$ where $\bar{U}_{\mf{x}}(t) = \min\{
    \max\{ U_{\mf{x}}(t), \ell_\mf{x}\}
    ,u_\mf{x}\}$. 
        
    Intervene on $\mf{X}_t = \mf{x}_t$ for round $t$ and receive reward $Y_t$ from $P_{\mf{x}_t}^\ast$. 
}
\caption{TRansport bounds Upper Confidence Bound (\textsc{trUCB})}
\label{alg: trUCB}

\end{algorithm}

\begin{corollary}
    The expected number of pulls, $\mathbb{E}N_t(\mf{x})$, is zero for all actions $\mf{x}$ satisfying $u_{\mf{x}} < \ell^\star$. 
\end{corollary}
\begin{proof}
   All actions $\mf{x}$ such that $u_{\mf{x}} < \ell^\star$ are removed from $\mc{I}^\ast$ in Line 4. Let $\bar{\mf{x}} = \arg\max_{\mf{x} \in \mc{I}^\ast} \ell_\mf{x}$. If $u_\mf{x} < \ell_{\bar{\mf{x}}}$, then for all rounds $t$, it holds that 
    $\bar{U}_{\mf{x}}(t) \leq u_{\mf{x}} < \ell_{\bar{\mf{x}}} \leq \bar{U}_{\bar{\mf{x}}}(t)$. Hence, any action $\mf{x}$ satisfying $u_{\mf{x}} < \ell^\star$ will never be selected throughout the learning process. 
\end{proof}

\section{Experiments}\label{sec: experiments}  
In this section, we present empirical results demonstrating that exploration over the action space $\mc{I}^\ast$ and the clipped index within the transport bounds $U_{\mf{x}}(t)$ lead to lower cumulative regret (CR). We compare \textsc{trUCB} (Alg.~\ref{alg: trUCB}) with standard UCB over all combinations of arms (UCB) and over POMISs (\textsc{poUCB}), focusing primarily on \textsc{poUCB} to ensure a fair evaluation of transportability. The number of trials is set to 50k, which is sufficient to observe the performance differences, as shown in \Cref{fig: experiment}. Further detailed explanations and settings regarding experiments are provided in \Cref{sec: app experimental details}.

 

\begin{wrapfigure}[6]{r}{0.18\textwidth}
    \vspace{-1em}
    \centering
        \begin{tikzpicture}[x=1cm,y=.8cm,>={Latex[width=1.4mm,length=1.7mm]},
        font=\sffamily\sansmath\tiny,
        line width=0.19mm,
        RR/.style={draw,circle,inner sep=0mm, minimum size=4.5mm,font=\sffamily\tiny},
        transp/.style={regime,fill=betterred,minimum size=1mm,draw=none,
        inner sep=1mm}, 
        regime/.style={rectangle,draw=betterred, thick,draw=betterred,minimum size=2.5mm},
        regime line/.style={draw=betterred, thick},
      regime induced/.style={draw=betterpurple, very thick},
      rotate = 0]
        \node[RR] (A) at (0,0) {$A$};
        \node[RR] (B) at (1,1) {$B$};
        \node[RR] (Y) at (2,0) {$Y$};
        \draw[->] (A) -- (Y);
        \draw[->] (B) -- (Y);

        \draw[->] (B) -- (A);

        \draw[dashed, <->] (A) to[bend right = 30] (Y);

        \node[transp, above=2.5mm of A] (AT)
        {};
        \draw[->,regime line] (AT) -- (A);

        \end{tikzpicture}
    \vspace{-1em}
    \caption{}
    \label{wrapfig: 3var selection}
\end{wrapfigure}
\paragraph{Task 1.} We start with a simple structural causal bandit problem where $\mf{N}^\ast = \emptyset$ represented by \Cref{wrapfig: 3var selection}. A decision maker has an RCT prior $\mb{Z}^1 = \{\{B\}\}$ from $\Delta^1 = \{A\}$. The action space without prior information corresponds to $\mb{P}_{\mc{G},Y}^{\mf{N}^\ast} = \{\{B\}, \{A,B\}\}$. In this setting, we have $\mb{E}_{P^\ast_b} Y = \sum_{a,y} y Q^\ast[A,Y]$ by \Cref{eq: decompose reward}. However, there is no c-factor satisfying \Cref{thm: causal bound} since $\{A\} \cap \Delta^1 \neq \emptyset$. In contrast, consider $\mb{E}_{P^\ast_{a,b}} Y = \sum_{y} y Q^\ast[Y] = \sum_{y} y Q^1[Y]$ according to $\{Y\} \cap \Delta^1 = \emptyset$ and we can find $Q^1[A,Y] = P^1_b(a,y)$. This implies that $\sum_{y} y  P^1_b(a,y) \leq \mb{E}_{P^\ast_{a,b}} Y \leq \sum_{y} y  \{P^1_b(a,y) + 1 - P^1_b(a) \}$. Using these expressions, the decision maker can estimate causal bounds for four actions corresponding to the POMISs: $do(A=0,B=0): [0.1675, 1]$, $ do(A=0,B=1): [0.2965, 0.7940]$, $ do(A=1,B=0): [0.8325, 1]$ and $ do(A=1,B=1): [0.2935, 0.7960]$. The algorithm computes the dominance bounds as $\ell^\star = 0.8325$ and $u^\star = \infty$. Among the four actions, the upper causal bounds $u_{do(A = 0,B = 1)}$ and $u_{do(A = 1,B = 1)}$ are lower than $\ell^\star$, leading them to be pruned by \textsc{trUCB}. Therefore, the algorithm begins the online interaction with the final action space $\mc{I}^\ast$ excluding these two actions. We observe that the mean cumulative regret at the final trial is $\boldsymbol{47.94}$ for \textsc{trUCB} and $\boldsymbol{130.16}$ for \textsc{poUCB}, which is $\frac{\text{CR for \textsc{trUCB}}}{\text{CR for \textsc{poUCB}}} = \boldsymbol{36.83\%}$ of the latter. 


\begin{wrapfigure}[8]{r}{0.19\textwidth}
    \centering
        \begin{tikzpicture}[x=1cm,y=.8cm,>={Latex[width=1.4mm,length=1.7mm]},
        font=\sffamily\sansmath\tiny,
        line width=0.19mm,
        RR/.style={draw,circle,inner sep=0mm, minimum size=4.5mm,font=\sffamily\tiny},
        transp/.style={regime,fill=betterred,minimum size=1mm,draw=none,
        inner sep=1mm}, 
        regime/.style={rectangle, draw=betterred,very thick, draw=betterred,minimum size=2.5mm},
        regime line/.style={draw=betterred, very thick},
      regime induced/.style={draw=betterpurple, very thick},
      rotate = -5]
        \node[RR,fill=gray!30] (A) at (0,0) {$A$};
        \node[RR] (B) at (0,2) {$B$};
        \node[RR,fill=gray!30] (C) at (1,1) {$C$};
        \node[RR] (Y) at (2,0) {$Y$};
        \draw[->] (A) -- (Y);
        \draw[->] (A) -- (C);
        \draw[->] (B) -- (C);
        \draw[->] (C) -- (Y);


        \node[transp, above=2.5mm of A] (AT)
        {};
        \draw[->,regime line] (AT) -- (A);
        

        \node[transp, above=2mm of B, draw = none, fill = betterblue] (BT2)
        {};
        \draw[->,regime line, draw = betterblue] (BT2) -- (B);
        
        \draw[dashed, <->] (C) to[bend left = 30] (B);
        \draw[dashed, <->] (B) to[bend left = 50] (Y);

        \end{tikzpicture}
    \vspace{-0.7em}
    \caption{}
    \label{wrapfig: 4var experiment selection}
\end{wrapfigure}
\paragraph{Task 2.} We consider the setting in \Cref{wrapfig: 4var experiment selection} where $\Delta^1 = \{A\}$ and $\Delta^2 = \{B\}$ with priors $\mb{Z}^1 = \{\emptyset\}$ and $\mb{Z}^2 = \{\{C\}\}$ and constraint $\mf{N}^\ast = \{A,C\}$. The initial action space corresponds to $\mb{P}_{\mc{G},Y}^{\mf{N}^\ast}$. The action $do(\emptyset)$ is transportable since $\mb{E}_{P^\ast_{\emptyset}}Y = \sum_{a,b,c,y}yQ^2[A]Q^1[B,C,Y] = \sum_{a,b,c,y}y P^2_c(a)P^1(b,c,y \mid a)  = 0.4844$. On the other hand, $do(b)$ is \textit{not} transportable, as $P^\ast_{b}(y) = \sum_{a,c}{\color{NavyBlue}Q^2[A]Q^2[Y]}Q^\ast[C]$. By \Cref{thm: causal bound}, $Q^\ast[C]$ is bounded by $Q^1[B,C]$, leading to $\ell_b = \sum_{a, c, y}y {\color{NavyBlue} P^2_{c}(a)P^2_{c}(y | a)}P^1(c | a,b)P^1(b)$. This expression yields $\ell_{do(B{=}0)} = 0.2097$ and $\ell_{do(B{=}1)} = 0.2752$. The upper causal bound for $do(b)$ is given by $\sum_y y \min\{1,\sum_{a, c} {\color{NavyBlue} P^2_{c}(a)P^2_{c}(y | a)}\{P^1(b) P^1(c | a,b) + 1 - P^1(b)\}\}$. We thus have the upper causal bounds, $u_{do(B{=}0)} = 0.6783$ and $u_{do(B{=}1)} = 0.8066$. The dominance bound is $[\ell^\star, u^\star] = [0.4844, 0.7697]$ with $u^\star = 0.7697$ derived from the expected reward of the transportable action $\mb{E}_{P^\ast_{do(A{=}1,C{=}1)}}Y = \sum_y yP^2_{do(C{=}1)}(y | A{=}1) = 0.7697$, which is an unconstrained POMIS. Since $u_{do(B=1)} {=} 0.8066 > 0.7697 {=} u^\star$, the final transport upper bound of $do(B{=}1)$ is updated to $u^\star$. The resulting transport bounds for $do(B{=}0)$ and $do(B{=}1)$ are $[0.2097, 0.6783]$ and $[0.2752, 0.7697]$, respectively. Although the size of the action space remains unchanged (i.e., no action is removed from $\mc{I}^\ast$, implying that the action spaces of all three algorithms are identical), we observe that accounting for the transport bounds improves performance, with CR of \textsc{trUCB} reduced to $\boldsymbol{38.6\%}$ of \textsc{poUCB}.

\begin{wrapfigure}[10]{r}{0.25\textwidth}
    \vspace{-.75em}
    \centering
        \begin{tikzpicture}[x=1cm,y=1cm,>={Latex[width=1.4mm,length=1.7mm]},
        font=\sffamily\sansmath\tiny,
        line width=0.19mm,
        RR/.style={draw,circle,inner sep=0mm, minimum size=4.5mm,font=\sffamily\tiny},
        transp/.style={regime,fill=betterred,minimum size=1mm,draw=none,
        inner sep=1mm}, 
        regime/.style={rectangle,draw=betterred,very thick,draw=betterred,minimum size=2.5mm},
        regime line/.style={draw=betterred, very thick},
      regime induced/.style={draw=betterpurple, very thick},
      rotate = -20]

        \node[RR, fill=gray!30] (W) at (0,0.2) {$W$};
        \node[RR] (X) at (1,0.7) {$X$};
        \node[RR] (Z) at (2,1) {$Z$};
        \node[RR] (Y) at (3,1) {$Y$};
        \node[RR, fill=gray!30] (T) at (0,2) {$T$};
        \node[RR] (R) at (1,2) {$R$};

        \draw[->] (W) -- (X);
        \draw[->] (X) -- (Z);
        \draw[->] (Z) -- (Y);
        \draw[->] (R) -- (Z);
        \draw[->] (T) -- (R);

        \draw[<->, dashed] (W) to [bend right=20] (X);
        \draw[<->, dashed] (W) to [bend left=30] (R);
        \draw[<->, dashed] (W) to [bend right=30] (Z);
        \draw[<->, dashed] (X) to [bend right=30] (Y);
        \draw[<->, dashed] (R) to [bend left=30] (Y);

    

        \node[transp, above=2.5mm of T, draw = betterred, fill = betterred] (TT)
        {};
        \draw[->,regime line, draw = betterred] (TT) -- (T);

        \node[transp, above=2.5mm of R, draw = betterblue, fill = betterblue] (RT)
        {};
        \draw[->,regime line, draw = betterblue] (RT) -- (R);

        \end{tikzpicture}
    \caption{}
    \label{wrapfig: correa selection}
\end{wrapfigure}
\paragraph{Task 3.} We consider a more involved scenario (\Cref{wrapfig: correa selection}) to validate our result. Let $\Delta^1 = \{T\}$ and $\Delta^2 = \{R\}$, with priors $\mb{Z}^1 = \{\emptyset, \{Z\}\}$ and $\mb{Z}^2 = \{\{Z\}\}$, and constraint $\mf{N}^\ast = \{T,W\}$. The algorithm starts by initializing the action space as $\mb{P}_{\mc{G},Y}^{\mf{N}^\ast} = \{\emptyset, \{R\}, \{X\}, \{Z\}\}$ where $do(\emptyset)$ and $do(z)$ are transportable while $do(x)$ and $do(r)$ are \textit{not}. In this setting, the transportable target POMIS action yields $\mb{E}_{P^\ast_{do(Z=1)}}Y = 1$, leading to $\ell^\star = u^\star = 1$. Furthermore, we obtain the upper causal bounds $u_{do(\emptyset)} = 0.5514$, $u_{do(X=0)} = 0.7901$ and $u_{do(Z=0)} = 0.034$, which are lower than $\ell^\star$. Consequently, the three actions are eliminated from $\mc{I}^\ast$ by the algorithm. We observe mean cumulative regrets of $\boldsymbol{173.62}$ (\textsc{poUCB}) and $\boldsymbol{95.69}$ (\textsc{trUCB}), achieving a $\boldsymbol{55.1\%}$ regret ratio. These results demonstrate performance improvements when transported causal knowledge from source environments is taken into account.

\begin{figure}[t]
    \hfill
    \begin{minipage}[b]{0.31\textwidth}\centering
        \subfloat[Task 1]{
            \centering
            \includegraphics[width=\textwidth]{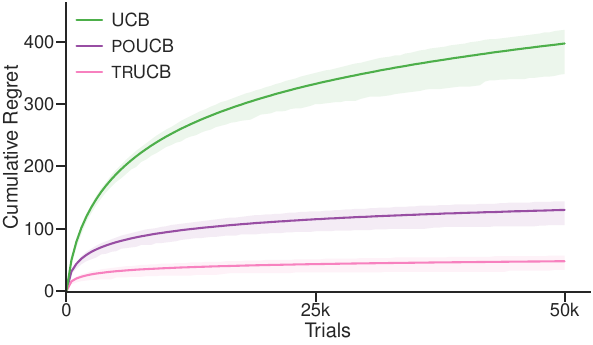}%
            }
    \end{minipage}
    \hspace{0.5em}
    \begin{minipage}[b]{0.31\textwidth}\centering
        \subfloat[Task 2]{
            \centering
            \includegraphics[width=\textwidth]{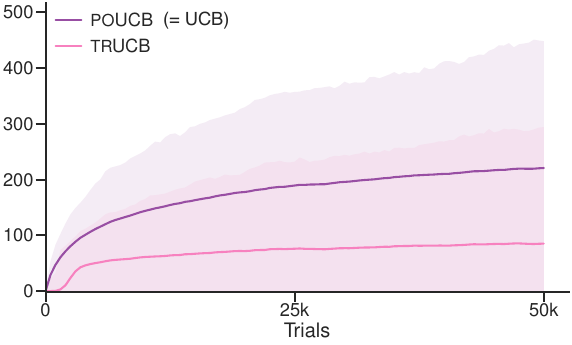}%
            }
    \end{minipage}
    \hspace{0.5em}
    \begin{minipage}[b]{0.31\textwidth}\centering
        \subfloat[Task 3]{
            \centering 
            \includegraphics[width=\textwidth]{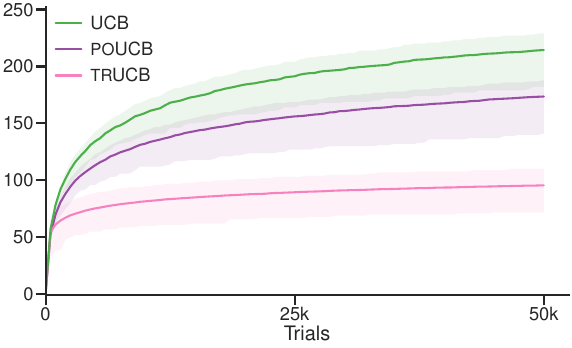}%
            }
    \end{minipage}
   \hfill\null
   \hfill
    \caption{Cumulative regrets of \textsc{trUCB} ({\color{betterpink}pink}) compared with standard UCB over all combinations of arms ({\color{bettergreen}green}) and over POMISs ({\color{betterpurple}purple}). Each simulation is repeated 1,000 times to ensure consistency, and the shaded regions indicate the 2.5th and 97.5th percentiles of the empirical cumulative regrets. }
    \label{fig: experiment}
\end{figure}

\section{Conclusion}  

We investigated a structured causal bandit strategy that can utilize prior data from related heterogeneous environments. Since source environments may differ, some knowledge of the underlying structure and potential discrepancies was necessary to enable consistent extrapolation. To address this, we proposed a strategy that exploits transportable causal knowledge by incorporating bounds equipped with dominance relations and causal structure. We demonstrated that the resulting bandit algorithm, leveraging causal knowledge, enjoys a sub-linear regret bound that depends on the extent of such knowledge. We believe that these results have practical implications for designing intelligent agents, providing a foundation for optimizing the action space when historical data is available.



\newpage

\begin{ack}
This work was supported by the IITP (RS-2022-II220953/50\%) and NRF (RS-2023-00211904/50\%) grant funded by the Korean government.
\end{ack}

\bibliographystyle{plainnat} 
\bibliography{reference}

\clearpage
\newpage
\appendix 

\renewcommand{\baselinestretch}{0.5}\normalsize
\tableofcontents
\renewcommand{\baselinestretch}{1.05}\normalsize
\newpage

\section{Related Works}\label{sec: app related works}

\paragraph{Identification.}  Causal effect identification, also known as the \textit{identification problem} \citep{pearl:95}, concerns whether the causal effect of an intervention on a set of variables can be uniquely computed from the observational distribution over observed variables and a given causal diagram. Foundational works \citep{tian2002general, shpitser2006identification,huang2006pearl,huang2008completeness} culminated in a complete graphical and algorithmic characterization of the problem. Beyond purely observational sources, there has been growing interest in generalizing the identification problem to settings where both observational and experimental data are available.  
\citet{bareinboim2012causal} studied the conditions under which the causal effect is uniquely computable, given a causal diagram and a collection of observational and experimental distributions over \textit{all} subsets of a given set. \citet{lee2019general} and \citet{kivva2022revisiting} investigated the identification problem over arbitrary collections of distributions and established necessary and sufficient graphical conditions for generalized identification.

\paragraph{Partial identification.}  
Given a causal diagram, one can express a causal effect in terms of the observational distribution using standard identification algorithms. However, challenges related to non-identifiability may arise, and the target effect may not be uniquely computable from observational data. The framework of \textit{partial identification} \citep{balke1995counterfactuals, balke1997bounds} addresses this issue by constraining the parameter space of causal effects within an interval. \citet{zhang2022partial} proposed a polynomial programming approach to solve partial identification problems for arbitrary causal diagrams.  Partial identification has also been applied in the causal decision-making literature to estimate dynamic treatment regimes \citep{zhang2019near, zhang2020designing, zhang2022online}, reinforcement learning \citep{zhang2024eligibility, bareinboim2024introduction}, bandit algorithms \citep{zhang2017transfer, zhang2021bounding, joshi2024towards}, and other domains \citep{jalaldoust2024partial, bellot2024towards, ruan2024causal}.

\paragraph{Transportability.}
In the causal inference literature, the problem of identifying causal effects under potential environment discrepancies has been extensively studied through the theory of \textit{transportability} \citep{pearl2011transportability, bareinboim2012transportability, bareinboim2013transportability, bareinboim2016causal}. These works focus on determining whether a causal effect can be identified across environments, and which aspects of causal knowledge can be transferred. \citet{lee2020general} investigated transportability under arbitrary combinations of experiments conducted in both the source and target environments. \citet{correa2020general} extended this framework to handle soft interventions, while \citet{correa2022counterfactual} further generalized it to the setting of counterfactual effects. More recently, \citet{jalaldoust2024partial} proposed a parameterization model approach \citep{zhang2022partial, xia2021causal} to bound non-transportable causal effects. \citet{zhang2017transfer} explored the integration of prior knowledge into the multi-armed bandit (MAB) framework under restrictive graph structures such as the Bow and IV settings. Building upon this, \citet{bellot2023transportability} extended these ideas to arbitrary causal diagrams, and \citet{deng2025transfer} further generalized them to settings where contextual information is available. 

\paragraph{Structural causal bandits.}  
\citet{lee2018structural} formalized the \textit{structural causal bandit} framework, in which a bandit instance is structured by an SCM, and each action corresponds to an intervention on a subset of variables. The authors proposed a sound and complete graphical characterization for identifying actions that could be part of an optimal strategy, enabling an agent to avoid unnecessary exploration \textit{a priori}, without any actual interaction. \citet{lee2019structural} extended the framework to settings involving non-manipulable variables. \citet{lee2020characterizing} established the framework under stochastic policies and demonstrated the informativeness of such policies. \citet{everitt2021agent} and \citet{carey2024toward} further investigated the completeness of the graphical characterization of optimal policy spaces, although the general completeness remains an open problem. \citet{wei2023approximate} proposed a parameterization-based approach to incorporate shared information among possibly-optimal actions. \citet{elahi2024partial} extended the SCM-MAB framework to settings where no causal graph is assumed to be accessible, requiring their algorithm to perform causal discovery---i.e., to construct the causal structure---during online interaction. Recently, \citet{park2025structural} investigates SCM-MAB in settings where the available information is not a full partial ancestral graph representing the Markov equivalence class of the true causal diagram. 

\section{Background}\label{sec: app background}

\paragraph{D-separation.} In a causal diagram $\mc{G}$, a path $p$ between vertices $X$ and $Y$ is a \textit{d-connecting} path relative to a set $\mf{Z}$ if (i) every non-collider on $p$ is not a member of $\mf{Z}$; and (ii) every collider on $p$ is an ancestor of some member of $\mf{Z}$. Two variables $\mf{X}$ and $\mf{Y}$ are said to be \textit{d-separated} by $\mf{Z}$ if there is no d-connecting path between $X$ and $Y$ relative to $\mf{Z}$. Two disjoint sets $\mf{X}$ and $\mf{Y}$ are said to be d-separated by $\mf{Z}$ if every variable in $\mf{X}$ is d-separated from every variable in $\mf{Y}$ by $\mf{Z}$ and denoted as $(\mf{X} \CI_d \mf{Y} \mid \mf{Z})_{\mc{G}}$. 
\paragraph{Do-calculus.} \citet{pearl:95} devised \textit{do-calculus} which acts as a bridge between observational and interventional distributions from a causal diagram without relying on any parametric assumptions. 
\begin{theorem}[Do-calculus \citep{pearl:95}]
    Let $\mc{G}$ be a causal diagram compatible with a structural causal model $\mc{M}$, with endogenous variables $\mf{V}$. For any disjoint $\mf{X}$, $\mf{Y}$, $\mf{W}$, $\mf{Z}$ $\subseteq \mf{V}$, the following rules are valid.
    
    {\small
    \begin{align*}
        &\textbf{ Rule 1. }
            P(\mf{y} \mid do(\mf{w}), \mf{x}, \mf{z}) = P(\mf{y} \mid do(\mf{w}), \mf{z})
                 &&\text{if }\mf{X} \text{ and } \mf{Y} \text{ are d-separated by }\mf{W} \cup \mf{Z} \text{ in } \mc{G}_{\overline{\mf{W}}}\\[.5em]
        &\textbf{ Rule 2. } 
            P(\mf{y} \mid do(\mf{w}), do(\mf{x}), \mf{z}) = P(\mf{y} \mid do(\mf{w}), \mf{x}, \mf{z}) 
                &&\text{if }\mf{X} \text{ and } \mf{Y} \text{ are d-separated by }\mf{W} \cup \mf{Z} \text{ in } \mc{G}_{\overline{\mf{W}},\underline{\mf{X}}}\\[.5em]
        &\textbf{ Rule 3. }
            P(\mf{y} \mid do(\mf{w}), do(\mf{x}), \mf{z}) = P(\mf{y} \mid do(\mf{w}), \mf{z}) 
                &&\text{if }\mf{X} \text{ and } \mf{Y} \text{ are d-separated by }\mf{W} \cup \mf{Z} \text{ in } \mc{G}_{\overline{\mf{W}},\overline{\mf{X(Z)}}}\\[.3em]
        &\text{ where } \mf{X(Z)} \triangleq \mf{X} \setminus \An(\mf{Z})_{\mc{G}[\mf{V} \setminus \mf{W}]}.
    \end{align*}}    
\end{theorem}
\paragraph{Latent projection.} The \textit{latent projection} \citep{verma1990equivalence} of a causal diagram $\mc{G}$ over $\mf{V}$ on $\mf{X}$, denoted by $\mc{G}\langle \mf{X} \rangle$ is a causal diagram over $\mf{X}$ such that, in addition to including edges in $\mc{G}[\mf{X}]$, for every pair of distinct vertices $V_i, V_j \in \mf{X}$, (i) add a directed edge $V_i \to V_j$ in $\mc{G}\langle \mf{X} \rangle$ if there exists a directed path from $V_i$ to $V_j$ in $\mc{G}$ such that every non-endpoint vertex on the path is not in $\mf{X}$, and (ii) add a bidirected edge $V_i \bedge V_j$ in $\mc{G}\langle \mf{X} \rangle$ if there exists a divergent path between $V_i$ and $V_j$ in $\mc{G}$ such that every non-endpoint vertex on the path is not in $\mf{X}$. 
\paragraph{Identification.}  Causal effect identification \citep{pearl:95} concerns whether the causal effect of an intervention on a set of variables can be uniquely computed from the observational distribution over observed variables and a given causal diagram. 

\begin{definition}[Identifiability \citep{pearl:2k}]\label{def: identification}
    The causal effect of the intervention on $\mf{X} = \mf{x}$ is \textit{identifiable} in $\mc{G}$, if for any two positive models $\mc{M}_1$ and $\mc{M}_2$ that induce the causal diagram $\mc{G}$, $P^{\mc{M}_1}(\mf{V}) = P^{\mc{M}_2} (\mf{V}) > 0$ implies $P_{\mf{x}}^{\mc{M}_1}(\mf{y}) = P_{\mf{x}}^{\mc{M}_2}(\mf{y})$. 
\end{definition}
In the identification problem, a basic structural unit known as the \textit{c-component (confounded component)} plays a crucial role. Given a semi-Markovian graph $\mc{G}$ over a set of variables $\mf{V}$, there exists a unique partition such that each subgraph of $\mc{G}$ is a c-component. 
\begin{definition}[C-component \citep{tian:02}]
    Let $\mc{G}$ be a semi-Markovian graph such that a subset of its bidirected arcs forms a spanning tree over all vertices in $\mc{G}$. Then $\mc{G}$ is a \textit{c-component}. 
\end{definition}
We denote by $\cc_\mc{G}$ the collection of maximal c-components so that $\cc_\mc{G} = \{\mf{C}_j\}_{j=1}^l$ implies that $\mf{C}_i$ is a maximal c-component, for each $\mf{C}_i \subseteq \mf{V}$, and there is no bidirected edge between $\mf{C}_i$ and $\mf{C}_j$ in $\mc{G}$ for $i \neq j$. Following \citet{tian:02}, for any $\mf{C} \subseteq \mf{V}$, we define function $Q[\mf{C}](\mf{v})$ = $P_{\mf{v} \setminus \mf{c}}(\mf{c})$. Moreover, $Q[\mf{V}](\mf{v}) = P(\mf{v})$ and $Q[\emptyset](\mf{v}) = 1$. For convenience, we omit input $\mf{v}$ and write $Q[\mf{C}]$.The importance of c-components lies in the following lemma. 
\begin{lemma}[Lemma 3 in \citet{tian:02}]\label{lem: ancestral}
    Let $\mf{C}\subseteq \mf{C}' \subseteq \mf{V}$, if $\mf{C}$ is an ancestral set in $\mc{G}[\mf{C}']$, then
    \begin{equation}
        \sum_{\mf{c'\setminus c}} Q[\mf{C}'] = Q[\mf{C}].
    \end{equation}
\end{lemma}
\begin{lemma}[Lemma 4 in \citet{tian:02}]\label{lem: Qdecompose}
    Let $\mf{C}\subseteq \mf{V}$, and assume that $\mf{C}$ is partitioned into c-components $\mf{C}_1, \cdots, \mf{C}_l$ in $\mc{G}[\mf{C}]$. Then,
    \begin{enumerate}[label=(\roman*)]
        \item $Q[\mf{C}]$ can be decomposed as
            \begin{equation}
                Q[\mf{C}] = \prod_{j=1}^l Q[\mf{C}_j]
            \end{equation}
        \item Let $\prec$ be a topological order over the variables in $\mf{C}$ according to $\mc{G}[\mf{C}]$ such that $C_1 \prec C_2 \cdots \prec C_k$. Let $\mf{C}^{\preceq i}$ be the variables in $\mf{C}$ that ordered before $C_i$ including $C_i$. Let $\mf{C}^{\succ i}$ be the variables in $\mf{C}$ that ordered after $C_i$.
        Then each $Q[\mf{C}_i]$ is computable from $Q[\mf{C}]$ and is given by
        \begin{equation}
            Q[\mf{C}_j] = \prod_{C_i \in \mf{C}_j}\frac{Q[\mf{C}^{\preceq i}]}{Q[\mf{C}^{\preceq i-1}]}
        \end{equation}
        where each $Q[\mf{C}^{\preceq i}] = \sum_{\mf{C}^{\succ i}} Q[\mf{C}]$.
        
        \item Each $\frac{Q[\mf{C}^{\preceq i}]}{Q[\mf{C}^{\preceq i-1}]}$ is a function only of $\mf{T}_i \triangleq \Pa(\bar{\mf{C}_i})_\mc{G} \setminus \bar{\mf{C}_i}$, where $\bar{\mf{C}_i}$ is the c-component of $\mc{G}[\mf{C}^{\preceq i}]$ that contains $C_i$.
    \end{enumerate}    
\end{lemma}
Let $\mf{D}_1, \cdots, \mf{D}_k$ be the c-components of $\mc{G}$. Then, for any $\mf{C}_j \in \cc_{\mc{G}\setminus \mf{X}}$, there exists $\mf{C}_j \subseteq \mf{D}_i$ since if the variables in $\mf{C}_j$ are connected by a bidirected path in a subgraph of $\mc{G}$, they must also be connected in $\mc{G}$. Each c-factor $Q[\mf{C}_j]$ is identifiable if it is computable from $Q[\mf{D}_i]$, which can be determined recursively by repeatedly applying \Cref{lem: ancestral,lem: Qdecompose}. Based on this recursive strategy, \citet{tian:02} proposed an identification algorithm that first decomposes the causal effect into a set of c-factors ${Q[\mf{C}_j]}$ and then checks the identifiability of each component iteratively. 
\paragraph{Bounding causal effect.} We now introduce the concept of \textit{natural bounds} \citep{manski1990nonparametric,robins1989analysis}, which are functions of the observational data that consistently bound the causal effect $P_{\mf{x}}(\mf{y})$, regardless of the underlying causal structure of the system. 
\begin{definition}[Natural bounds \citep{manski1990nonparametric,robins1989analysis}]\label{def: natural bound}
    The natural bounds for a causal effect $P_{\mf{x}}(\mf{y})$ are given by
    \begin{align}
        P(\mf{x},\mf{y}) \leq P_{\mf{x}}(\mf{y}) \leq P(\mf{x},\mf{y}) + 1 - P(\mf{x}).
    \end{align}
\end{definition}

\section{Possibly Optimal Intervention Sets: Characterizations}\label{sec: app POIS}
In this section, we provide graphical characterizations of POMIS and POIS, accompanied by illustrative examples. When given a causal diagram $\mc{G}$, \textit{minimal unobserved confounders' territory} (MUCT) and \textit{interventional border} (IB) provide a graphical characterization of PO(M)IS. 

\begin{definition} [Unobserved-confounders' territory \citep{lee2018structural}] \label{def: MUCT for cd}
    Let $\mc{H} = \mc{G}[\An(Y)_{\mc{G}}]$. A set of variables $\mf{T} \subseteq \mf{V}(\mc{H})$ containing $Y$ is called a \textit{UC-territory} on $\mc{G}$ with respect to $Y$ if $\De(\mf{T})_\mc{H} = \mf{T}$ and $\mathsf{CC}(\mf{T})_{\mc{H}} = \mf{T}$. If there is no $\mf{T}' \subsetneq \mf{T}$, we refer to it as a \textit{minimal UC-territory} (MUCT) denoted as $\MUCT(\mc{G},Y)$. 
\end{definition}

\begin{definition} [Interventional border \citep{lee2018structural}] \label{def: IB for cd}
    Let $\mf{T}$ be a minimal UC-territory on causal diagram $\mc{G}$ with respect to $Y$. Then $\Pa(\mf{T})_{\mc{G}} \setminus \mf{T}$ is called an \textit{interventional border} (IB) for $\mc{G}$ with respect to $Y$ denoted as $\IB(\mc{G},Y)$. 
\end{definition}

MUCT is the minimal set of variables that is closed under both descendants and bidirected connections; and IB consists of the parents of MUCT, excluding MUCT itself. Intuitively, MUCT is the minimal closed mechanism that conveys all hidden information from unobserved confounders to the downstream reward, while IB consists of the nodes that directly affect this closed mechanism. 

\begin{theorem}[Theorem 6 in \citet{lee2018structural}]\label{thm6 in lee2018}
    Given information $\langle \mc{G},Y \rangle$, a set $\mf{X} \subseteq \mf{V} \setminus \{Y\}$ is a POMIS with respect to $\langle \mc{G},Y \rangle$ if and only if $\IB(\mc{G}_{\overline{\mf{X}}},Y) = \mf{X}$. 
\end{theorem}

Following the established structures, we provide a characterization for POIS with respect to $\langle \mc{G},Y \rangle$. 

\begin{proposition}[Graphical characterization of POIS]\label{prop: characterization POIS}
    Let $\mf{T}_\mf{X} \triangleq \MUCT(\mc{G}_{\overline{\mf{X}}},Y)$. 
    A set $\mf{X} \subseteq \mf{V} \setminus \{Y\}$ is a POIS with respect to $\langle \mc{G},Y \rangle$ if and only if $\IB(\mc{G}_{\overline{\mf{X}}},Y) \subseteq \mf{X} \subseteq \An(\mf{T}_{\mf{X}})_{\mc{G}} \setminus \mf{T}_{\mf{X}}$. Moreover, if $\mf{X} = \IB(\mc{G}_{\overline{\mf{X}}},Y)$, then $\mf{X}$ is a POMIS with respect to $\langle \mc{G},Y \rangle$.  
\end{proposition}
\begin{proof}
    Note that $\IB(\mc{G}_{\overline{\mf{X}}},Y) = \Pa(\mf{T}_{\mf{x}})_{\mc{G}}\setminus \mf{T}_{\mf{X}}$ is a POMIS with respect to $\langle \mc{G},Y \rangle$ by Prop.~5 in \citep{lee2018structural}. If $\mf{X} = \IB(\mc{G}_{\overline{\mf{X}}},Y)$, then $\mf{X}$ is a POMIS according to \Cref{thm6 in lee2018}. Now, we consider $\IB(\mc{G}_{\overline{\mf{X}}},Y) \subsetneq \mf{X} \subseteq \An(\mf{T}_{\mf{X}})_{\mc{G}} \setminus \mf{T}_{\mf{X}}$ We will show that there is no directed path from any $X \in \mf{X} \setminus \IB(\mc{G}_{\overline{\mf{X}}},Y)$ to $Y$ that does not pass through $\IB(\mc{G}_{\overline{\mf{X}}},Y)$, which implies $\mb{E}_{P_{\mf{x}}}Y = \mb{E}_{P_{\IB,\mf{x} \setminus \IB}}Y=
    \mb{E}_{P_{\IB}}Y=
    \mb{E}_{P_{\mf{x}[\IB]}}Y$. For the sake of contradiction, suppose that there exists a directed path from $X$ to $Y$ in $\mc{G}$ that does not pass through any node in $\IB(\mc{G}_{\overline{\mf{X}}},Y)$. Since $Y \in \mf{T}_{\mf{X}}$, its parent must belong to either $\mf{T}_{\mf{X}}$ or $\IB(\mc{G}_{\overline{\mf{X}}},Y)$. This implies that $X \in \mf{T}_{\mf{X}}$ while also $X \in \An(\mf{T}_{\mf{X}})_{\mc{G}} \setminus \mf{T}_{\mf{X}}$, which leads to a contradiction. 
\end{proof}
For example, consider the causal diagram in \Cref{fig: 5var e} where $\{B,D\}$ is a POMIS with respect to $\langle \mc{G},Y \rangle$ since $\MUCT(\mc{G}_{\overline{\{B,D\}}},Y) = \{Y\}$ and $\IB(\mc{G}_{\overline{\{B,D\}}},Y) = \{B,D\}$, which satisfies \Cref{thm6 in lee2018}. Furthermore, $\{A,B,D\}$, $ \{C,B,D\}$ and $\{A,B,C,D\}$ are POISs with respect to $\langle \mc{G},Y \rangle$. Moreover, they are equivalent to the POMIS $\{B,D\}$. 
\begin{corollary}[Equivalence]\label{cor: equivalence}
    Let $\mf{R} \subseteq \mf{V} \setminus \{Y\}$ be a POIS with respect to $\langle \mc{G}, Y \rangle$ and $\mf{R}^\dagger \triangleq \IB(\mc{G}_{\overline{\mf{R}}},Y)$ denote the corresponding POMIS. Then, $\mf{R}$ and $\mf{R}^\dagger$ are equivalent in terms of expected reward.
\end{corollary}
\begin{proposition}[Sharing transportability]\label{prop: pois share MUCT}
    Let $\mf{R} \subseteq \mf{V} \setminus \{Y\}$ be a POIS with respect to $\langle \mc{G},Y \rangle$. Then $\mf{R}^\dagger \triangleq \IB(\mc{G}_{\overline{\mf{R}}},Y)$ is a POMIS with respect to $\langle \mc{G},Y \rangle$. Moreover, they share a causal bound; given $\langle \mc{G}^{\boldsymbol{\Delta}}, \mb{Z} \rangle$, $\ell_{\mf{R}} = \ell_{\mf{R}^\dagger}$ and $u_{\mf{R}} = u_{\mf{R}^\dagger}$. 
\end{proposition}
\begin{proof}
    This follows from the proof of \Cref{prop: characterization POIS}. An important observation is that they share the same MUCT, $\MUCT(\mc{G}_{\overline{\mf{R}}},Y)$, and that there is no path from $\mf{R}$ to $Y$ that does not pass through $\mf{R}^\dagger$. Therefore, $\mb{E}_{P^\ast_{\mf{r}[\mf{R}^\dagger]}} Y = \mb{E}_{P^\ast_{\mf{r}^\dagger}} Y \sum_{y} y P^\ast_{\mf{r}^\dagger}(y) = \sum_{\MUCT} y \prod_{q=1}^mQ^\ast[\mf{C}_q]$ (where $\mf{Y}^{\boldsymbol{+}} = \MUCT(\mc{G}_{\overline{\mf{R}}},Y)$). The proof is thereby concluded by the application of \Cref{thm: causal bound}. 
\end{proof}


\begin{figure}[!t]
    \centering  
    \begin{minipage}[b]{0.19\textwidth}\centering  
    \subfloat[$\mc{G}$]{
        \begin{tikzpicture}[x=0.8cm, y = 1.3cm,>={Latex[width=1.4mm,length=1.7mm]},
        font=\sffamily\sansmath\tiny,
        line width=0.2mm,
        RR/.style={draw,circle,inner sep=0mm, minimum size=4.5mm,font=\sffamily\tiny}]
        \node[RR] (C) at (0,1) {$C$};
        \node[RR] (D) at (1,0.7) {$D$};
        \node[RR] (Y) at (1.5,0.15) {$Y$};
        \node[RR] (B) at (1.8,1.2) {$B$};
        \node[RR] (A) at (0.7,1.5) {$A$};

        \draw[->] (A) -- (B);
        \draw[->] (A) -- (C);
        \draw[->] (C) -- (D);
        \draw[->] (D) -- (Y);
        \draw[->] (B) -- (D);
        \draw[->] (B) -- (Y);
        \draw[dashed, <->] (C) to[bend right = 30] (Y);
        \end{tikzpicture}}
    \end{minipage} \hfill
    \begin{minipage}[b]{0.19\textwidth}\centering  
    \subfloat[$\mc{G} = \Gi{\emptyset}$]{
        \begin{tikzpicture}[x=0.8cm, y = 1.3cm,>={Latex[width=1.4mm,length=1.7mm]},
        font=\sffamily\sansmath\tiny,
        line width=0.2mm,
        RR/.style={draw,circle,inner sep=0mm, minimum size=4.5mm,font=\sffamily\tiny}]
        \node[RR] (C) at (0,1) {$C$};
        \node[RR] (D) at (1,0.7) {$D$};
        \node[RR] (Y) at (1.5,0.15) {$Y$};
        \node[RR] (B) at (1.8,1.2) {$B$};
        \node[RR] (A) at (0.7,1.5) {$A$};

        \draw[->] (A) -- (B);
        \draw[->] (A) -- (C);
        \draw[->] (C) -- (D);
        \draw[->] (D) -- (Y);
        \draw[->] (B) -- (D);
        \draw[->] (B) -- (Y);

        \begin{pgfonlayer}{background}
        \draw[betterblue!30, fill=betterblue!30, line width=6mm,line cap=round,line join=round] 
        (Y.center) to[bend left = 30] (C.center) -- (D.center) -- (Y.center) -- (D.center) --cycle;
        \end{pgfonlayer}
        
        \draw[dashed, <->] (C) to[bend right = 30] (Y);
        
        \end{tikzpicture}}
    \end{minipage}\hfill
    \begin{minipage}[b]{0.19\textwidth}\centering  
    \subfloat[$\mc{G}_{\overline{\{B,C\}}}$]{
        \begin{tikzpicture}[x=0.8cm, y = 1.3cm,>={Latex[width=1.4mm,length=1.7mm]},
        font=\sffamily\sansmath\tiny,
        line width=0.2mm,
        RR/.style={draw,circle,inner sep=0mm, minimum size=4.5mm,font=\sffamily\tiny}]
        \node[RR] (C) at (0,1) {$C$};
        \node[RR] (D) at (1,0.7) {$D$};
        \node[RR] (Y) at (1.5,0.15) {$Y$};
        \node[RR] (B) at (1.8,1.2) {$B$};
        \node[RR] (A) at (0.7,1.5) {$A$};

        \draw[->, draw = gray!10] (A) -- (B);
        \draw[->, draw = gray!10] (A) -- (C);
        \draw[dashed, <->, draw = gray!10] (C) to[bend right = 30] (Y);
        
        \draw[->] (C) -- (D);
        \draw[->] (D) -- (Y);
        \draw[->] (B) -- (D);
        \draw[->] (B) -- (Y);

        \begin{pgfonlayer}{background}
        \draw[betterblue!30, fill=betterblue!30, line width=6mm,line cap=round,line join=round] 
        (Y.center) --cycle;
        \end{pgfonlayer}
        
        \end{tikzpicture}}
    \end{minipage}\hfill
    \begin{minipage}[b]{0.19\textwidth}\centering  
    \subfloat[$\Gi{\{A,B\}}$]{
        \begin{tikzpicture}[x=0.8cm, y = 1.3cm,>={Latex[width=1.4mm,length=1.7mm]},
        font=\sffamily\sansmath\tiny,
        line width=0.2mm,
        RR/.style={draw,circle,inner sep=0mm, minimum size=4.5mm,font=\sffamily\tiny}]
        \node[RR] (C) at (0,1) {$C$};
        \node[RR] (D) at (1,0.7) {$D$};
        \node[RR] (Y) at (1.5,0.15) {$Y$};
        \node[RR] (B) at (1.8,1.2) {$B$};
        \node[RR] (A) at (0.7,1.5) {$A$};

        \draw[->, draw = gray!10] (A) -- (B);
        \draw[->] (A) -- (C);
        \draw[-> ] (C) -- (D);
        \draw[-> ] (D) -- (Y);
        \draw[->] (B) -- (D);
        \draw[->] (B) -- (Y);

        \begin{pgfonlayer}{background}
        \draw[betterblue!30, fill=betterblue!30, line width=6mm,line cap=round,line join=round] 
        (Y.center) to[bend left = 30] (C.center) -- (D.center) -- (Y.center) -- (D.center) --cycle;
        \end{pgfonlayer}
        
        \draw[dashed, <->] (C) to[bend right = 30] (Y);
        \end{tikzpicture}}
    \end{minipage} \hfill
    \begin{minipage}[b]{0.19\textwidth}\centering  
    \subfloat[$\Gi{\{B,D\}}$]{
        \begin{tikzpicture}[x=0.8cm, y = 1.3cm,>={Latex[width=1.4mm,length=1.7mm]},
        font=\sffamily\sansmath\tiny,
        line width=0.2mm,
        RR/.style={draw,circle,inner sep=0mm, minimum size=4.5mm,font=\sffamily\tiny}]
        \node[RR] (C) at (0,1) {$C$};
        \node[RR] (D) at (1,0.7) {$D$};
        \node[RR] (Y) at (1.5,0.15) {$Y$};
        \node[RR] (B) at (1.8,1.2) {$B$};
        \node[RR] (A) at (0.7,1.5) {$A$};

        \draw[->, draw = gray!10] (A) -- (B);
        \draw[->] (A) -- (C);
        \draw[->, draw = gray!10] (C) -- (D);
        \draw[->] (D) -- (Y);
        \draw[->, draw = gray!10] (B) -- (D);
        \draw[->] (B) -- (Y);

        \begin{pgfonlayer}{background}
        \draw[betterblue!30, fill=betterblue!30, line width=6mm,line cap=round,line join=round] 
        (Y.center) --cycle;
        \end{pgfonlayer}
    
        \draw[dashed, <->] (C) to[bend right = 30] (Y);
        \end{tikzpicture}\label{fig: 5var e}}
    \end{minipage} \hfill
    \caption{The blue region illustrates MUCT. (b, c) are non-POMIS examples, while (d, e) correspond to POMIS. (e) Subsets $\{B, D\} \subseteq \mf{R} \subseteq \{A, B, C, D\}$ are POIS and share the same expected reward.}
    \label{fig: 5var}
\end{figure}
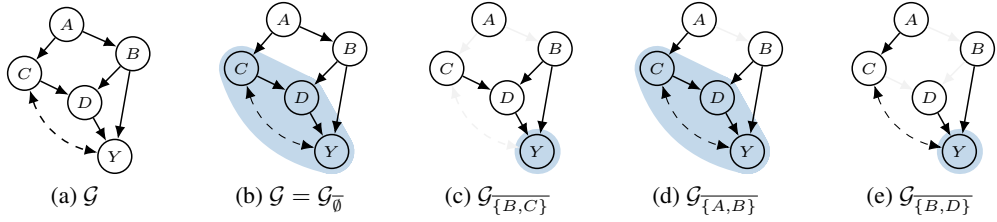

\begin{proposition}[Theorem 4 in \citet{lee2019structural}] \label{prop: thm4 in lee2019}
    Given $\langle \mc{G}, Y, \mf{N} \rangle $, we have $\mb{P}_{\mc{G},Y}^\mf{N} = \mb{P}_{\mc{H},Y}$ where $\mc{H} = \mc{G}\langle \mf{V} \setminus \mf{N} \rangle$ is the latent projection of $\mc{G}$ onto $\mf{V} \setminus \mf{N}$.    
\end{proposition}

\subsection*{Algorithmic Characterization of POIS} 
The algorithm \textsf{POISs} (Alg.~\ref{alg: pois}) is identical to \textsf{POMISs} (Alg.~1 in \citet{lee2018structural}), except for Lines 3 and 9, where the set $\{\mf{X}\}$ is replaced with $\{\mf{R} \mid \mf{X} \subseteq \mf{R} \subseteq \An(\mf{T})_{\mc{G}} \setminus \mf{T}\}$ in order to include  not only the POMIS $\mf{X}$ but also all POISs equivalent to it. The algorithm completely enumerates all POISs avoiding redundant computations by Thm.~9 in \citet{lee2018structural} and it takes $\mc{O}(kn^2)$ where $k$ denotes the number of POIS and $n = \lvert \mf{V} \rvert$. 

\begin{algorithm}[h]
\small
\SetAlgoNoEnd

\DontPrintSemicolon
\SetKwInput{KwInput}{Input}

\textbf{function } \textsf{POISs}($\mc{G}$, $Y$)\;

\Indp{
$\mf{T}$, $\mf{X}$ = $\MUCT(\mc{G},Y)$, $\IB(\mc{G},Y)$; $\mc{Q} = \mc{G}_{\overline{\mf{X}}}[\mf{T} \cup \mf{X}]$

\Return{\rm $\{\mf{R}  \mid \mf{X} \subseteq \mf{R} \subseteq \An(\mf{T})_{\mc{G}}  \setminus \mf{T}\} \cup$ \textsf{subPOISs}$(\mc{Q},Y, \textsf{reversed(topological-sort}(\mc{Q}), \emptyset$})\;
}

\Indm
\BlankLine 
\textbf{function }{\rm \textsf{subPOISs}$(\mc{G},Y,\boldsymbol{\pi}, \mf{O})$}\;

\Indp{
    $\mf{P} = \emptyset$

    \For{\rm $\pi[i] \in \boldsymbol{\pi}$}{
        $\mf{T}$, $\mf{X}$, $\boldsymbol{\pi'}$, $\mf{O}'$ = $\MUCT(\mc{G}_{\overline{\pi[i]}},Y)$, $\IB(\mc{G}_{\overline{\pi[i]}},Y)$, $\boldsymbol{\pi}[i+1:]$, $\mf{O} \cup \pi[1:i-1]$

        \If{\rm $\mf{X}\cap \mf{O}' = \emptyset$}{
            $\!\mf{P} = \mf{P} \cup \{\mf{R}  \mid \mf{X} \subseteq \mf{R} \subseteq \An(\mf{T})_{\mc{G}} {\setminus} \mf{T} \} {\cup}$ \textsf{subPOISs}$(\mc{G}_{\overline{\mf{X}}}[\mf{T} {\cup} \mf{X}],Y, \boldsymbol{\pi'}, \mf{O}'$) \textbf{if} $\boldsymbol{\pi'} \neq \emptyset$ \textbf{else} $\emptyset$. 
            }
        \Return{$\mf{P}$}
        }
    }
\caption{Algorithm enumerating all POISs.}
\label{alg: pois}
\end{algorithm}

\begin{proposition}
    The algorithm \textsf{POISs} (Alg.~\ref{alg: pois}) returns all, and only POISs given $\langle \mc{G}, Y \rangle$.
\end{proposition}
\begin{proof}
    This follows from Thm. 9 in \citet{lee2018structural}) and \Cref{prop: characterization POIS}.
\end{proof}

\section{Dominance Bounds}\label{sec: app dominance bound}
This section presents technical details relevant to the computation of dominance bounds. 

\subsection{Lower Dominance Bound}

We start by demonstrating that the $\MISs$ (Alg.~3 in \citet{lee2018structural}) is sound and complete under constraints. 
\begin{lemma}\label{lem: MISs valid non-manipulable}
    The algorithm $\MISs(\mc{G} \langle \mf{V} \setminus \mf{N} \rangle , Y)$ returns all and only MISs with respect to $\langle \mc{G},Y,\mf{N} \rangle$.  
\end{lemma}
\begin{proof}
Let $\mf{X}$ be an MIS with respect to $\langle \mc{G} \langle \mf{V} \setminus \mf{N} \rangle,Y\rangle$. We will prove $\mf{X}$ is also an MIS with respect to $\langle \mc{G},Y \rangle$ by proving the contrapositive. Suppose that $\mf{X}$ is \textit{not} an MIS with respect to $\langle \mc{G}, Y \rangle$. Then, there exists a node $X \in \mf{X}$ such that there is no \textit{proper directed path}\footnote{We refer to a directed path from $X \in \mf{X}$ to $Y$ as a \textit{proper} directed path with respect to $\mf{X}$ if only the first node $X$ belongs to $\mf{X}$.} from $X$ to $Y$ with respect to $\mf{X}$ in $\mc{G}$.
That is, every directed path from $X$ to $Y$ in $\mc{G}$ forms $X \to V_1 \to \cdots \to V_n \to Z \to W_1 \to \cdots \to W_m \to Y$ with $n,m \geq 0$ for an arbitrary $Z \in \mf{X} \setminus \{X\}$. Since the latent projection does not introduce any directed edges from $V_i$ to $W_j$ in $\mc{G}\langle \mf{V} \setminus \mf{N} \rangle$, all such paths correspond to non-proper directed paths from $X$ to $Y$ with respect to $\mf{X}$ in $\mc{G}\langle \mf{V} \setminus \mf{N} \rangle$. Therefore, $\mf{X}$ is \textit{not} an MIS with respect to $\langle \mc{G} \langle \mf{V} \setminus \mf{N} \rangle,Y\rangle$, which completes the contrapositive. 
Now, suppose $\mf{X}$ is an MIS with respect to $\langle \mc{G},Y \rangle$. By definition, there is no proper subset $\mf{X}' \subseteq \mf{V} \setminus \{Y\}$ such that $\mf{X}'$ is equivalent to $\mf{X}$. Moreover, since $\mf{X}$ is defined over $\mf{V} \setminus \{Y\} \setminus \mf{N}$, any such proper subset $\mf{X}'$ must also be defined over the same superset. Hence, any set $\mf{X}$ included in the output of  $\MISs(\mc{G} \langle \mf{V} \setminus \mf{N} \rangle , Y)$ is an MIS with respect to $\langle \mc{G},Y,\mf{N} \rangle$. Combining this with the soundness and completeness of the $\MISs$ algorithm, the proof is complete. 
\end{proof}
%
Equipped with this result and the definition of MIS (\Cref{def: mis}), the lower dominance bound $\ell^\star$ can be computed using $\MISs$; $\ell^\star=\max_{\mf{w} \in \Omega_{\mf{W}},\mf{W} \in \MISs(\mc{G}\langle \mf{V} \setminus \mf{N}^\ast \rangle , Y)}$. 

\subsection{Upper Dominance Bound}

We proceed to the upper dominance bound. It is crucial to identify a space in which POMISs under different constraints can be meaningfully compared. The following lemma shows that POMISs remain MISs under weaker constraints.
\begin{lemma} \label{lem: MIS to MIS}
    If $\mf{X}$ is a (PO)MIS with respect to $\langle \mc{G}, Y, \mf{N} \rangle$, then it is also an MIS with respect to $\langle \mc{G}, Y \rangle$. 
\end{lemma}
\begin{proof}
    In the proof of \Cref{lem: MISs valid non-manipulable}, we have shown that if $\mf{X}$ be an MIS with respect to $\langle \mc{G} \langle \mf{V} \setminus \mf{N} \rangle,Y\rangle$, then it is also an MIS with respect to $\langle \mc{G},Y \rangle$. Since the set of POMISs is a subset of the set of MISs under the same constraint, the result also holds for the POMISs, which concludes the proof. 
\end{proof}
\begin{lemma}\label{lem: subconstraint dominance}
   Let $\mb{P}_{\mc{G},Y}^{\mf{N}'}$ and $\mb{P}_{\mc{G},Y}^{\mf{N}}$ be a set of POMISs corresponding to the constraints $\mf{N}'$ and $\mf{N}$, respectively, where $\mf{N}' \subseteq \mf{N}$. Then $\mb{P}_{\mc{G},Y}^{\mf{N}}$ dominates $\mb{P}_{\mc{G},Y}^{\mf{N}'}$.
\end{lemma}
\begin{proof}
    Let $\mc{H} \triangleq \mc{G} \langle \mf{V} \setminus \mf{N}' \rangle$. According to \Cref{prop: thm4 in lee2019}, we have $\mb{P}_{\mc{G},Y}^{\mf{N}'} = \mb{P}_{\mc{H},Y}^{\mf{N} \setminus \mf{N}'}$ and $\mb{P}_{\mc{G},Y}^{\mf{N}} = \mb{P}_{\mc{H},Y}$.  Let us denote $\mf{N}^\dagger = \mf{N} \setminus \mf{N}'$. According to \Cref{lem: MIS to MIS}, any $\mf{X} \in \mb{P}_{\mc{H},Y}^{\mf{N}^\dagger}$ is an MIS with respect to $\langle \mc{H},Y \rangle$. Since POMISs dominate MISs under the same constraint $\mf{N}^\dagger$, the claim follows. 
\end{proof}

\thmDominanceRelations*
\begin{proof}
    Without loss of generality, we assume $\mf{N} = \emptyset$ since we can equivalently reformulate the graph as $\mc{G} = \mc{G} \langle \mf{V} \setminus \mf{N} \rangle$ and the constraint set as $\mf{N}^\ast = \mf{N}^\ast \setminus \mf{N}$. According to \Cref{lem: MIS to MIS}, any $\mf{X} \in \mb{P}_{\mc{G},Y}^{\mf{N}^\ast}$ is an MIS with respect to $\langle \mc{G},Y \rangle$. Furthermore, for any unconstrained POISs $\mf{R}$, we can always find an equivalent unconstrained POMIS $\mf{R}^\dagger = \IB(\mc{G}_{\overline{\mf{R}}},Y)$ that yields the same expected reward, i.e., $\mb{E}_{P_{\mf{r}}}Y = \mb{E}_{P_{\mf{r}[\mf{R}^\dagger]}}Y$, as supported by \Cref{cor: equivalence}. Therefore, the upper dominance bound is sound. The lower bound also holds by the definition of MIS (\Cref{def: mis}) and POIS (\Cref{def: POIS}).  
\end{proof}

Building on the statements, we present the algorithm \textsc{udb} (Alg.~\ref{alg: UDB}) which returns a valid upper dominance bound $u^\star$ given the inputs $(\mc{G}, Y, \mf{N}, \mb{U})$ where $\mb{U}$ denotes a collection of upper causal bounds for all actions. 

\begin{algorithm}[h]
\small
\SetAlgoNoEnd

\SetKwInput{KwInput}{Input}

\DontPrintSemicolon
\textbf{function }{\rm \textsc{udb}$(\mc{G}, Y, \mf{N}, \mb{U})$}:

\setcounter{AlgoLine}{0}

\KwInput{$\mc{G}$: causal diagram; $Y$: reward variable; $\mf{N}$: non-manipulable variables; $\mb{U} = \{u_{\mf{x}} \mid \mf{x} \in \Omega_{\mf{X}}, \mf{X} \in 2^{\mf{V} \setminus \{Y\}}\}$: collection of upper causal bounds. }

\textbf{Ensure:}~{All upper causal bounds $u_\mf{x} \in \mb{U}$ have been computed.}

    

        
        

            

        



\BlankLine

\Indp
{
    
    Compute the latent projection $\mc{H} = \mc{G} \langle \mf{V} \setminus \mf{N} \rangle $; Initialize the upper dominance bound $u^\star = 0$. 

    \lFor{\rm $\mf{R} := \textsf{POMISs}(\mc{H},Y)$}{
        \textbf{if} $u^\star > \max_{\mf{r} \in \Omega_{\mf{R}}}u_{\mf{r}}$ \textbf{then} $(u^\star, \mf{R}^\star) = (\max_{\mf{r} \in \Omega_{\mf{R}}} u_{\mf{r}} ,\mf{R})$. 
    }
    
    \lIf{\rm $u^\star < \infty$ and $u_{\mf{r}^\star} = \ell_{\mf{r}^\star}$ (i.e., transportable) }{\Return{$u^\star$} 

    \lIf{$\mf{N} = \emptyset$}{\Return{$u^\star$}}

    \lFor{\rm all $\mf{N}' \subset \mf{N}$ such that $\vert \mf{N}' \vert = \vert \mf{N} \vert - 1$}{$u^\star = \min\{u^\star,\textsc{udb}(\mc{G}, Y, \mf{N}', \mb{U})\}$}
        \Return{$u^\star$}
    }
}

\caption{Upper Dominance Bound (\textsc{udb})}
\label{alg: UDB}
\end{algorithm}

In Line 3, the algorithm attempts to compute upper causal bounds of POMISs. If the algorithm reaches Line 6, this implies that (1) there exists at least one POMIS action whose upper causal bound is $\infty$, or (2) the current upper dominance bound corresponds to upper causal bound of a \textit{non}-transportable action. In such cases, the algorithm recursively explores weaker constraints $\mf{N}'$ and returns the tightest bound among the results of recursive calls; on the other hand, if $u^\star < \infty$ and the upper dominance bound corresponds to the expected reward of a transportable action (Line 4), no further recursive calls are required---this is justified by \Cref{lem: subconstraint dominance}---and the algorithm simply returns $u^\star$. 


%
\paragraph{Runtime analysis.} In the worst case, the algorithm may need to traverse up to the unconstrained POMISs, resulting in an exponential time complexity in the size of $\mf{N}^\ast$, i.e., $\mc{O}(2^{\vert \mf{N}^\ast \vert})$.  
\paragraph{Dominance relationship example.} As a concrete example to illustrate the dominance relationships in \Cref{wrapfig: hierarchy}, we present an SCM that exhibits these relations. We consider an SCM $\mc{M}$ where 

\begin{align}
    \mc{M}
    = \begin{cases}
            \mf{U} &= \{U_A,U_B,U_C,U_Y,U_{BC},U_{BY}\}\\
            \mf{V} &= \{A,B,C,Y\} \\
            \mf{F} &= 
            \begin{cases}
                f_{A} = u_{AB}, f_{B} = u_B \oplus u_{BC} \oplus u_{BY}, \\
                f_C = a \oplus b \land u_C, f_Y = u_A \oplus u_C \oplus u_Y \land u_{BY}
            \end{cases}\\
            P(\mf{U}) &= 
            \begin{cases}
                U_A \sim \Bern(0.3), U_B \sim \Bern(0.1), U_C \sim \Bern(0.25),\\
                U_Y \sim \Bern(0.2), U_{BC} \sim \Bern(0.2), U_{BY} \sim \Bern(0.15).
            \end{cases}
    \end{cases}
\end{align}
 
To elaborate on dominance relations, let us suppose access to the causal bounds of expected rewards of all actions. In this setting, the value of upper dominance bound is:

(d) $u^\star = \max_{\mf{r} \in \Omega_{\mf{R}}, \mf{R} \in \mb{P}_{\mc{G},Y}^{\{A,C\}} = \{\emptyset, \{B\}\} } \mb{E}_{P^\ast_{\mf{r}}} Y  \thickapprox 0.265$. For weaker constraints, we have (b) $\max_{\mf{r} \in \Omega_{\mf{R}}, \mf{R} \in \mb{P}_{\mc{G},Y}^{\{A\}}=\{\emptyset, \{B\}, \{C\}\} } \mb{E}_{P^\ast_{\mf{r}}} Y \thickapprox 0.782$, (c) $\max_{\mf{r} \in \Omega_{\mf{R}}, \mf{R} \in \mb{P}_{\mc{G},Y}^{\{C\}} = \{\emptyset, \{A\}, \{A,B\}\} } \mb{E}_{P^\ast_{\mf{r}}} Y \thickapprox 0.782$; and (a) $\max_{\mf{r} \in \Omega_{\mf{R}}, \mf{R}^\dagger \in \mb{P}_{\mc{G},Y}= \{\emptyset, \{A\}, \{A,C\}\} } \mb{E}_{P^\ast_{\mf{r}}} Y \thickapprox 0.97$. 

Therefore, we observe the dominance relationships: (d) $\leq$ (b, c) $\leq$ (a). 

\section{Causal Bounds: Partial Transportability}\label{sec: app partial trans}

This section presents technical details relevant to the computation of causal bounds (\Cref{prop: Q trans bound,thm: causal bound} in \Cref{sec: trans}). We begin by introducing counterfactual variables, which play a useful role in our proofs.

\paragraph{Counterfactual variables.} Given a set of variables $\mf{Y} \subseteq \mf{V}$, the solution for $\mf{Y}$ in $\mc{M}_\mf{x}$ defines a \textit{potential response} for a unit $\mf{u}$, denoted as $\mf{Y}_\mf{x}(\mf{u})$. Averaging over the space of $\mf{U}$, a potential response $\mf{Y}_{\mf{x}}(\mf{u})$ induces a \textit{counterfactual variables} $\mf{Y}_{\mf{x}}$ \citep{pearl:2k}. 

\citet{correa2025counterfactual} introduced a novel calculus over probability quantities may defined at the counterfactual level, called the \textit{ctf-calculus}. The independence rule (Rule 2) in ctf-calculus requires the construction of another graphical object, known as the \textit{Ancestral Multi-World Network} (AMWN), which serves to identify d-separation \citep{pearl:95} relations among counterfactual variables.
\begin{theorem}[Counterfactual calculus (ctf-calculus); Theorem 3.1 in \citet{correa2025counterfactual}]
    Let $\mc{G}$ be a causal diagram, then for $\mf{Y}$, $\mf{X}$, $\mf{Z}$, $\mf{W}$, $\mf{T}$, $\mf{R} \subseteq \mf{V}$, the following rules hold for the probability distributions generated by any model compatible with $\mc{G}$:
    {
    \setlength{\belowdisplayskip}{10pt}%
    \setlength{\abovedisplayskip}{10pt}%
    \setlength{\jot}{0pt}
    \begin{align*}
        &\textbf{ Rule 1. (Consistency rule - Obs./intervention exchange)}
            \\  &\qquad 
                P(\mf{y}_{\mf{T}_\ast \mf{x}}, \mf{x}_{\mf{T}_\ast}, \mf{w}_\ast) = P(\mf{y}_{\mf{T}_\ast}, \mf{x}_{\mf{T}_\ast}, \mf{w}_\ast)
                & \quad
                \\[.5em]
        &\textbf{ Rule 2. (Independence rule - Adding/removing counterfactual observations)} 
            \\ &\qquad
            P(\mf{y}_{\mf{r}} \mid \mf{x}_\mf{t}, \mf{w}_\ast) = P(\mf{y}_{\mf{r}}, \mf{w}_\ast) \\
            &\qquad\qquad \text{if } (\mf{Y}_{\mf{r}} \CI \mf{X}_\mf{t} \mid \mf{W}_\ast)_{\mc{G}_A} 
            \\[.5em]
        &\textbf{ Rule 3. (Exclusion Rule - Adding/removing interventions)}
        \\ & \qquad 
        P(\mf{y}_{\mf{xz}}, \mf{w}_\ast) = P(\mf{y}_{\mf{z}}, \mf{w}_\ast)
        \\ &\qquad \qquad \text{if } \mf{X} \cap \An(\mf{Y}) = \emptyset \text{ in } \mc{G}_{\overline{\mf{Z}}}
    \end{align*}}   
    where $\mc{G}_A$ is the AMWN $\mc{G}_A(\mc{G}, \mf{Y}_\mf{r} \cup \mf{X}_\mf{t} \cup \mf{W}_\ast)$. 
\end{theorem}
However, since only Rule 1 (\textbf{R1}) is used in this paper, we refer the reader to \citet{correa2025counterfactual} for further details about AMWN.
%
 
\propQTransBound*
\begin{proof}
    We prove ours build on the notion of counterfactual variables. We use $P(\mf{y}_{\mf{x}})$ for probabilities $P(\mf{Y}_{\mf{x}} = \mf{y})$. By definition, $P_{\mf{x}}(\mf{y}) = P(\mf{y}_{\mf{x}})$. According to \Cref{lem: ancestral}, $\An(\mf{C})_{\mc{G}[\mf{C}']} = \mf{C}$ implies $Q^i[\mf{C}] = \sum_{\mf{c}' \setminus \mf{c}}Q^i[\mf{C}']$. Further, we have that 
    \setlength{\belowdisplayskip}{2pt}
    \setlength{\abovedisplayskip}{2pt}%
    \setlength{\jot}{-2pt}
    \begin{align*}
        Q^\ast[\mf{C}] = P^{\boldsymbol{\Delta}}_{\mf{v} \setminus \mf{c}}(\mf{c} \mid \mf{s}^{\mf{i}} = \boldsymbol{0}, \mf{s}^{-\mf{i}} = \boldsymbol{0}) = P^{\boldsymbol{\Delta}}_{\mf{v} \setminus \mf{c}}(\mf{c} \mid \mf{s}^{\mf{i}} = \mf{i}, \mf{s}^{-\mf{i}} = \boldsymbol{0}) = Q^i[\mf{C}]
    \end{align*}
    holds by $\mf{C} \cap \Delta^i = \emptyset$, and thus we have $Q^\ast[\mf{C}] = \sum_{\mf{c}' \setminus \mf{c}}Q^i[\mf{C}']$. Otherwise $\An(\mf{C})_{\mc{G}[\mf{C}']} \neq \mf{C}$, with out loss of generality, we have $\An(\mf{C})_{\mc{G}[\mf{C}']} = \mf{C}'$. Let $\mf{T} \triangleq \mf{C}' \setminus \mf{C}$, i.e., $\mf{C}' = \mf{T} \cup \mf{C}$. By basic probabilistic algebra, 
    \setlength{\belowdisplayskip}{2pt}
    \setlength{\abovedisplayskip}{2pt}%
    \setlength{\jot}{-2pt}
    \begin{align}
        Q^\ast[\mf{C}] &= Q^i[\mf{C}]
        =
        P^i(\mf{c}_{\mf{v} \setminus \mf{c}}) 
        = \sum_{\mf{t}'} P^i(\mf{c}_{\mf{v} \setminus \mf{c}}, \mf{t}'_{\mf{v} \setminus \mf{c}'})
        = \sum_{\mf{t}'} P^i(\mf{c}_{\mf{v} \setminus \mf{c}',\mf{t}}, \mf{t}'_{\mf{v} \setminus \mf{c}'})\\
        & \geq  P^i(\mf{c}_{\mf{v} \setminus \mf{c}',\mf{t}}, \mf{t}_{\mf{v} \setminus \mf{c}'})
        \overset{\textbf{R1}}{=} P^i(\mf{c}_{\mf{v} \setminus \mf{c}'}, \mf{t}_{\mf{v} \setminus \mf{c}'}) = Q^i[\mf{C}'].
    \end{align}
    
    We thus have $Q^\ast[\mf{C}] \geq Q^i[\mf{C}']$. Similarly, we prove the upper bound of $Q^\ast[\mf{C}]$. Using basic probabilistic algebra, 
    \setlength{\belowdisplayskip}{0pt}%
    \setlength{\abovedisplayskip}{0pt}%
    \begin{align}
        Q^\ast[\mf{C}] &= Q^i[\mf{C}]
        = P^i(\mf{c}_{\mf{v} \setminus \mf{c}}) \\[0.8em]
        &= \sum_{\mf{t}'} P^i(\mf{c}_{\mf{v} \setminus \mf{c}}, \mf{t}'_{\mf{v} \setminus \mf{c}'}) \\
        &\overset{\textbf{R1}}{=} P^i(\mf{c}_{\mf{v} \setminus \mf{c}',\mf{t}}, \mf{t}_{\mf{v} \setminus \mf{c}'}) + \sum_{\mf{t}' \neq \mf{t}} P^i(\mf{c}_{\mf{v} \setminus \mf{c}',\mf{t}}, \mf{t}'_{\mf{v} \setminus \mf{c}'})
    \end{align}
    \setlength{\belowdisplayskip}{0pt}%
    \setlength{\abovedisplayskip}{-2pt}%
    \begin{align}
        \phantom{Q^\ast[\mf{C}]} &\leq P^i(\mf{c}_{\mf{v} \setminus \mf{c}',\mf{t}}, \mf{t}_{\mf{v} \setminus \mf{c}'}) + 
        \sum_{\mf{t}' \neq \mf{t}} P^i(\mf{t}'_{\mf{v} \setminus \mf{c}'})\\
        &\overset{\textbf{R1}}{=}  P^i(\mf{c}_{\mf{v} \setminus \mf{c}'}, \mf{t}_{\mf{v} \setminus \mf{c}'}) + 1 - P^i(\mf{t}_{\mf{v} \setminus \mf{c}'})\\[0.7em]
        &=  P^i(\mf{c}_{\mf{v} \setminus \mf{c}'}, \mf{t}_{\mf{v} \setminus \mf{c}'}) + 1 - 
            \sum_{\mf{c}} P^i(\mf{c}_{\mf{v} \setminus \mf{c}'},\mf{t}_{\mf{v} \setminus \mf{c}'})\\
        &= Q^i[\mf{C}'] + 1 - \sum_{\mf{c}}Q^i[\mf{C}'].
    \end{align}
    We thus have $Q^\ast[\mf{C}] \leq Q^i[\mf{C}'] + 1 - \sum_{\mf{c}}Q^i[\mf{C}']$ which concludes the proof.  
\end{proof}
\vspace{0em}


\thmCausalBound*
\begin{proof}
    Recall \Cref{eq: decompose reward}, $\mb{E}_{P^\ast_{\mf{x}}} Y = \sum_{y} y P^\ast_{\mf{x}}(y) = \sum_{\mf{y}^{\boldsymbol{+}} } y \prod_{q=1}^mQ^\ast[\mf{C}_q]$. This expression can be decomposed into \textit{transportable} and \textit{non-transportable} components as follows: $\mb{E}_{P^\ast_{\mf{x}}} Y =\sum_{\mf{y}^{\boldsymbol{+}}} y {\color{NavyBlue}\prod_{q=1}^kQ^{i_q}[\mf{C}_q]}\prod_{q=k+1}^{m}Q^{\ast}[\mf{C}_q]$ where the first product represents the transportable terms. According to \Cref{prop: Q trans bound}, each non-transportable c-factor $Q^\ast[\mf{C}_q]$ satisfies the inequality $Q^i[\mf{C}_q'] \leq Q^\ast[\mf{C}_q] \leq Q^i[\mf{C}_q'] + 1 - \sum_{\mf{c}_q}Q^i[\mf{C}_q']$. Therefore, by substituting this inequality into the non-transportable components of the decomposition. Moreover, note that both \Cref{eq: decompose reward} and \Cref{eq: casual lower bound,eq: casual upper bound} are functions of $\mf{X} \cup \mf{Y}^{\boldsymbol{+}}$ where
    $\mf{Y}^{\boldsymbol{+}} = \An(Y)_{\mc{G}_{\underline{\mf{X}}}}$.
    Therefore, the expressions for $\ell_\mf{x}$ and $u_{\mf{x}}$ are valid. Moreover, taking the maximum or minimum over $\mf{Y}^{\boldsymbol{+}}$, with respect to $\mc{P}_{\mb{Z}}^{\Pi}$, returns valid bounds, equipped with $\min\{1, \cdot\}$. 
\end{proof}


\subsection{Partial Transportability of Causal Effects}

\begin{definition}[Partial-transportability] Let $\mc{G}^{\boldsymbol{\Delta}}$ be a collective selection diagram with respect to $\Pi = \{\pi_1, \cdots, \pi_n\}$ with a target domain $\pi^\ast$. Let $\mb{Z} = \{\mb{Z}^i\}_{i=1}^n$ be a specification of actions $\mb{Z}^i$ conducted in source environment $\pi^i$. We say that $P^\ast_\mf{x}(y)$ is  \textit{partially transportable} with respect to $\langle \mc{G}^{\boldsymbol{\Delta}}, \mb{Z} \rangle$ if it determines a bound $[\ell,u]$ for $P^\ast_\mf{x}(y)$ that is strictly contained in $[0,1]$ and valid over $\mc{P}^{\Pi}_\mb{Z} = \{P_\mf{z}^i \mid \mf{z} \in \Omega_{\mf{Z}
}, \mf{Z} \in \mb{Z}^i \in \mb{Z} \}$ in any collection of models that induce $\mc{G}^{\boldsymbol{\Delta}}$. 
\end{definition}

We propose the partial-transportability algorithm \textsc{paTR} (Alg.~\ref{alg: paTR}) which returns expressions of bounds for $P^\ast_{\mf{x}}(y)$. The prior specification $\mb{Z}$ and the corresponding distributions $\mc{P}_\mb{Z}^{\Pi}$ are defined globally and do not change with the specific invocation of the algorithm. In contrast, variables $\mf{V}$ and selection variables $\mf{S}$ reflect graph $\mc{G}$ and discrepancies $\boldsymbol{\Delta}$, respectively, relative to the arguments passed to the current execution of the procedure. 

In the algorithm, Line 5 breaks down the query into queries where $\mf{Y}$ in each sub-query forms a c-component. Line 6 examines whether some experimental distribution $P^i_{\mf{z}} \in \mc{P}_\mb{Z}^{\Pi}$  can be used to identify the query. If valid, \textsc{paTR} passes the query to \textsc{paID} with a slight modification of it and graph, taking into account the shared intervention between $\mf{Z}$ and $\mf{X}$. 

\textit{Remark.~} When the subroutine \textsc{paID} is called, the following conditions are satisfied (related to Lines 4 and 5): $(\mf{V} \setminus \mf{X}) \setminus \An(\mf{Y})_{\mc{G}_{\overline{\mf{X}}}} = \emptyset$ and $\mc{G} \setminus \mf{X}$ is a c-component. 

\textsc{paTR} enumerates \textit{all} expressions through the subroutine \textsc{paID}, which returns lower or upper bound depending on input $\mathsf{type} \in \{\ell,u\}$. The condition of \Cref{thm: causal bound} is inherently captured within \textsc{paTR}.  

\begin{algorithm}[!t]
\small
\SetAlgoNoEnd
\DontPrintSemicolon
\SetKwInput{KwInput}{Input}

\setcounter{AlgoLine}{0}



    



\textbf{function }{\rm \textsc{paTR}$(\mf{y}, \mf{x}, \mc{G}, \boldsymbol{\Delta}, \mathsf{type})$}\;

\KwInput{$\mf{y}$; $\mf{x}$; $\mc{G}$: causal diagram; $\boldsymbol{\Delta}$: discrepancies; $\mathsf{type} \in \{\ell,u\}$: type of bound. }


\Indp{
\lIf {\rm $\exists \mf{Z} \in \mb{Z}^i \in \mb{Z}$ such that $(\mf{X} = \mf{Z} \cap \mf{V}) \land (\mf{S}^i \CI_d \mf{Y})_{\mc{G}^{\boldsymbol{\Delta}}\setminus \mf{X}}$}{
    \textbf{yield} {\rm $P^i_{\mf{z}\setminus\mf{V},\mf{x}\cap \mf{Z}}(\mf{y})$.}}

\lIf {\rm $\mf{V}' := \mf{V} \setminus \An(\mf{Y})_{\mc{G}} \neq \emptyset$}{
\textbf{yield} {\rm \textsc{paTR}$(\mf{y}, \mf{x} \setminus \mf{V}', \mc{G}\setminus \mf{V}', \{\Delta^i \setminus \mf{V}' \mid \Delta^i \in \boldsymbol{\Delta} \})$}.}

\lIf {\rm $\mf{V}' := (\mf{V} \setminus \mf{X}) \setminus \An(\mf{Y})_{\mc{G}_{\overline{\mf{X}}}} \neq \emptyset$}{
\textbf{yield} {\rm \textsc{paTR}$(\mf{y}, \mf{x} \cup \mf{V}', \mc{G}, \boldsymbol{\Delta})$}.}

\lIf {\rm $\vert \cc_{\mc{G} \setminus \mf{X}} \vert > 1$}{
\textbf{yield} {\rm $\sum_{\mf{v} \setminus (\mf{y} \cup \mf{x})} \prod_{\mf{C} \in \cc_{\mc{G} \setminus \mf{X}}}$ \textsc{paTR}$(\mf{c}, \mf{v} \setminus \mf{c}, \mc{G}, \boldsymbol{\Delta})$}.}

\For{\rm $\pi^i \in \Pi$ such that $\Delta^i \cap (\mf{V} \setminus \mf{X}) = \emptyset$, \textbf{for} $\mf{Z} \in \mb{Z}^i$ such that $\mf{Z} \cap \mf{V} \subseteq \mf{X}$}
{


    \textbf{yield} {\rm \textsc{paID}$(\mf{y},\mf{x}\setminus \mf{Z}, P^i_{\mf{z}\setminus\mf{V},\mf{x}\cap \mf{Z}}, \mc{G} \setminus (\mf{Z} \cap \mf{X}), \mathsf{type})$}
    
    }
}

\BlankLine 
\Indm
\textbf{function }{\rm \textsc{paID}$(\mf{y}, \mf{x}, P, \mc{G}, \mathsf{type})$}\;
    
    \Indp{
    
    $\{\mf{C}\} \gets \cc_{\mc{G} \setminus \mf{X}}$. 
    
    \lIf{\rm $\mf{X} = \emptyset$}{\textbf{yield} {$\sum_{\mf{v}\setminus \mf{y} } P(\mf{v})$}.}

    \lIf{\rm $\mf{V} \neq \An(\mf{Y})_{\mc{G}}$}{\textbf{yield} {\textsc{paID}$(\mf{y}, \mf{x} \cap \An(\mf{Y})_{\mc{G}}, \sum_{\mf{v} \setminus \An(\mf{Y})_{\mc{G}}}P(\mf{v}), \mc{G}[\An(\mf{Y})_{\mc{G}}], \mathsf{type})$}.}

    \textbf{if} $\cc_{\mc{G}} = \{\mf{V}\}$ \textbf{then;}
        \lIf{$\mathsf{type} = \ell$}{
            \textbf{yield} {$\sum_{\mf{v} \setminus (\mf{y} \cup \mf{x})} P(\mf{v})$} \textbf{else} \textbf{yield} {$\sum_{\mf{v} \setminus (\mf{y} \cup \mf{x})}\{P(\mf{v}) + 1 - \sum_{\mf{v} \setminus \mf{x}} P(\mf{v})\}$.}
        }

    \lIf{\rm $\mf{C} \in \cc_{\mc{G}}$}{$\sum_{\mf{C} \setminus \mf{y}} \prod_{V_i \in \mf{Y}}P(v_i \mid \mf{v}_\pi^{(i-1)})$.}

    \lIf{\rm $\mf{C} \subsetneq \mf{C}' \in \cc_{\mc{G}}$}{\textbf{yield} {\textsc{paID}$(\mf{y}, \mf{x} \cap \mf{C}', \prod_{V_i \in \mf{C}'}P(V_i \mid \mf{V}_{\pi}^{(i-1)} \cap \mf{C}', \mf{v}_{\pi}^{(i-1)} \setminus \mf{C}'), \mc{G}[\mf{C}'], \mathsf{type})$}.}
    
    }
    
\caption{Partial-transportability algorithm (\textsc{paTR}).}
\label{alg: paTR}
\end{algorithm}



\begin{proposition}[Soundness]\label{prop: sound patr}
    \textsc{paTR} (Alg.~\ref{alg: paTR}) returns all expressions of upper (lower) bounds of $P_{\mf{x}}^\ast(y)$. 
\end{proposition}
\begin{proof}
     Let a superscript $l$ denote variables and values local to the function. The soundness of the algorithm was partially proved by \citet{lee2020general} under the case where the given query is transportable. Our main interest is the case where $P^\ast_{\mf{x}}(y)$ is non-transportable but holds the condition in \Cref{prop: Q trans bound} is satisfied. First, we show that $Q^i[\mf{V}_l] \leq P^\ast_{\mf{x}_{l}}(\mf{v}_l \setminus \mf{x}_l) \leq Q^i[\mf{V}_l] + 1 - \sum_{\mf{v}_l \setminus \mf{x}_l}Q^i[\mf{V}_l]$ at Line 12. Note that $P^\ast_{\mf{x}_{l}}(\mf{v}_l \setminus \mf{x}_l) = P^i_{\mf{x}_{l}}(\mf{v}_l \setminus \mf{x}_l)$ holds by Lem.~1 in \citet{lee2020general}. Moreover, it holds that $\mf{V}_l = \An(\mf{Y}_l)_{\mc{G}_l}$ and $\mf{V}_l$ is a c-component in $\mc{G}_l$, as ensured by earlier steps in \textsc{paID}. Therefore, we have $Q[\mf{V}_l] = P^i(\mf{v}_l) \leq P^\ast_{\mf{x}_{l}}(\mf{v}_l \setminus \mf{x}_l) \leq P^i(\mf{v}_l) + 1 - \sum_{\mf{v}_l \setminus \mf{x}_l} P^i(\mf{v}_l) = Q^i[\mf{V}_l] + 1 - \sum_{\mf{v}_l \setminus \mf{x}_l}Q^i[\mf{V}_l]$ due to \Cref{prop: Q trans bound}. Taking $P^\ast_{\mf{x}_l}(\mf{y}_l) = \sum_{\mf{v}_l \setminus (\mf{x}_l \cup \mf{y}_l)} P^\ast_{\mf{x}_l}(\mf{v}_l \setminus \mf{x}_l)$ concludes the proof. 
\end{proof}

Any value of \textsc{paTR}$(\mathsf{type}=\ell)$ is always less than or equal to $P^\ast_{\mf{x}}(y)$. However, for \textsc{paTR}$(\mathsf{type}=u)$, each non-transportable term of the form $\big\{ Q^{j_{q}}[\mf{C}'_{q}]  +1 - \sum_{\mf{c}_{q}}Q^{j_{q}}[\mf{C}'_{q}] \big \}$ is not  a single probability, which can lead to $\textsc{paTR}(\mathsf{type}=u)$ being greater than one. 

\begin{proposition}
    Let $\ell = \textsc{paTR}(y, \mf{x}, \mc{G}, \boldsymbol{\Delta}, \ell)$  and $u = \min \{1, \textsc{paTR}(y, \mf{x}, \mc{G}, \boldsymbol{\Delta}, u)\}$. Then $P_{\mf{x}}^\ast(y)$. $[\ell,u]$ is partially-transportable with respect to $\langle \mc{G}^{\boldsymbol{\Delta}}, \mb{Z} \rangle$ with the bounds given by $[\ell, u]$. 
\end{proposition}
\begin{proof}
    It has been established in \Cref{prop: sound patr} that the expressions returned by \textsc{paTR}$(\mathsf{type}=\ell)$ and \textsc{paTR}$(\mathsf{type}=u)$ are sound. Since $0 < \ell \leq P_{\mf{x}}^\ast(y)$, the lower bound $\ell$ lies within the unit interval. Moreover, since the use of the $\min$ operator ensures that $P_{\mf{x}}^\ast(y) \leq u \leq 1$, the upper bound $u$ lies within $[0, 1]$. Thus, the interval $[\ell, u]$ is strictly contained in $[0, 1]$, constituting a valid partially-transportable bound. 
\end{proof}

\paragraph{Runtime analysis.} The algorithm \textsc{paTR} runs in $\mc{O}(zn^4)$ where $z$ is the number of experiments $\lvert \mb{Z} \rvert$ and $n$ denote the number of vertices $\lvert \mf{V} \rvert$. Lines 3--5 are satisfied only once until the given query is fully factorized. Each factorized sub-query subsequently encounters Lines 3 and 4 once before proceeding to Line 6. The procedure \textsc{paID} can be invoked at most $z$ times. Each call to \textsc{paID} runs in $\mc{O}(n^3)$ time. The procedure \textsc{paTR} may be called up to $\mc{O}(n)$ times due to the factorization step in Line 6. Each invocation of \textsc{paTR} may call \textsc{paID} up to $z$ times, resulting in a total of $\mc{O}(zn)$ \textsc{paID} calls. Since each of these calls may trigger recursion up to $n$ times, the total number of recursive steps remains bounded. Assuming that set and graphical operations take $\mc{O}(n^2)$ time, the overall runtime is $\mc{O}(zn^4)$.

\subsection{Causal Bounds of Expected Rewards} 

In this section, we present the algorithm associated with \Cref{thm: causal bound}, which maximizes and minimizes over the sources to compute the tightest possible causal bounds, as demonstrated in Alg.~\ref{alg: causal bound}. 

\begin{algorithm}[h]
\small
\SetAlgoNoEnd
\DontPrintSemicolon
\SetKwInput{KwInput}{Input}

\setcounter{AlgoLine}{0}


\textbf{function }{\rm \textsc{CausalBound}$(y, \mf{x}, \mc{G}, \boldsymbol{\Delta},\mb{Z}, \mc{P}_{\mb{Z}}^{\Pi}$})\;

\KwInput{$y$: reward; $\mf{x}$: action; $\mc{G}$: causal diagram; $\boldsymbol{\Delta}$: discrepancies; $\mb{Z}$ experiments; $\mc{P}_{\mb{Z}}^{\Pi}$: distributions.}

\BlankLine 

\Indp{

Initialize causal bounds $[\ell_\mf{x},u_{\mf{x}}]$ as $[0,\infty)$. 

\lFor{\rm $P_{\mf{x}}^\ell(y):=\textsc{paTR}(\{y\},\mf{x},\mc{G},\boldsymbol{\Delta},\ell)$}{
    $\ell_{\mf{x}} = \max\{\ell_\mf{x},   \sum_{y} y\cdot P^\ell_{\mf{x}}(y)$\}.
}

\lFor{\rm $P^u_{\mf{x}}(y):= \textsc{paTR}(\{y\},\mf{x},\mc{G},\boldsymbol{\Delta},u)$}{
    $u_{\mf{x}} = \min\{u_\mf{x}, 
    \sum_{y} y\cdot \min\{1, P^u_{\mf{x}}(y)$\}\}.}
 }   

\Return{\rm $[\ell_{\mf{x}},u_{\mf{x}}]$.}

\caption{Causal bounds algorithm (\textsc{CausalBound}).}
\label{alg: causal bound}
\end{algorithm}


In Lines 3 and 4, the algorithm derives valid expressions for bounds of $P^\ast_\mf{x}(y)$ from $\textsc{paTR}(\mathsf{type} = \ell)$ and $\textsc{paTR}(\mathsf{type} = u)$, and updates the final bounds using $\mc{P}_{\mb{Z}}^{\Pi}$ to obtain tighter estimates. Specifically, if the estimated upper bound exceeds one, we replace it with one to maintain valid probabilistic bounds.

\section{Regret Upper Bounds of trUCB}
%

\begin{lemma}[Hoeffding's inequality \citep{hoeffding1994probability}]\label{lem: Hoef}
    Suppose $Y_1, \cdots, Y_n$ are independent random variables such that $Y_i \in [0,1]$ with $a_i < b_i$ for all $i$. Then, the following holds: 
    \setlength{\belowdisplayskip}{10pt}%
    \setlength{\abovedisplayskip}{10pt}%
    \setlength{\jot}{0pt}
    \begin{align*}
        P(\lvert \bar{Y}_n - \mu \rvert \geq \epsilon) \leq 2e^{-2n\epsilon^2}
    \end{align*}
    Therefore, setting $\delta = 2e^{-2n\epsilon^2}$, we get $\bar{Y}_n \leq \mu + \sqrt{\frac{\ln(1 \slash \delta)}{2n}}$ with probability $1-\delta$. 
\end{lemma}


\begin{restatable}[Regret bound]{proposition}{proptrUCB}\label{prop: trUCB}
Let $Y$ be the reward variable supported on $[0,1]$. Then, the cumulative regret of \textsc{trUCB} in the target SCM $\mc{M}^\ast$ after $T > 1$ rounds is bounded as
    \setlength{\belowdisplayskip}{2pt}%
    \setlength{\abovedisplayskip}{2pt}%
    \begin{align}
        R_T \leq 
        8 \sum_{\mf{x} : \Delta_\mf{x}> 0,\, u_\mf{x} \geq \mb{E}_{P^\ast_{\mf{x}^{\star}}} Y} \frac{\log(T)}{\Delta_\mf{x}} + \left(1 + \frac{\pi^2}{3}\right) \sum_{\mf{x} : \Delta_\mf{x} > 0, u_\mf{x} \geq \ell^\star} \Delta_\mf{x}. 
    \end{align}
\end{restatable}
\begin{proof}
    The proof follows the arguments of \citet{auer2002finite,zhang2017transfer,Zhang2023towards}. For convenience, we denote the optimal expected reward in the target environment $\mb{E}_{P^\ast_{\mf{x}^\star}}Y$ by $\mu^\star$. Let $\ell^\star \leq u_\mf{x} < \mu^\star$. Then, we have 
    \setlength{\belowdisplayskip}{2pt}
    \setlength{\abovedisplayskip}{2pt}%
    \setlength{\jot}{-2pt}
    \begin{align}
        N_T(\mf{x}) 
        &= \sum_{t=1}^T \mathds{I} \{ \mf{X}_t = \mf{x}\} \\
        &\leq 
        \sum_{t=1}^T \mathds{I} 
        \{ \bar{U}_{\mf{x}^\star}(t) < \mu^\star\} + 
        \mathds{I} 
        \{ \bar{U}_{\mf{x}^\star}(t) \geq \mu^\star, \mf{X}_t = \mf{x}\}\\
        & =
        \sum_{t=1}^T \mathds{I} 
        \{ \bar{U}_{\mf{x}^\star}(t) < \mu^\star\} + 
        \mathds{I} 
        \{ \bar{U}_\mf{x}(t) \geq \mu^\star\}.
    \end{align}
    The last step holds since an arm $\mf{x}$ is played at time step $t$ if and only if the causal clipped UCB $\bar{U}_\mf{x}(t)$ is maximal. Since $\ell_\mf{x} \leq \bar{U}_\mf{x}(t) \leq u_{\mf{x}}$ and $u_{\mf{x}} < \mu^\star$, we have $\bar{U}_\mf{x}(t) < \mu^\star$. Thus, 
    \setlength{\belowdisplayskip}{2pt}
    \setlength{\abovedisplayskip}{2pt}%
    \setlength{\jot}{-2pt}
    \begin{align}
        N_T(\mf{x}) 
        & \leq
        \sum_{t=1}^T \mathds{I} 
        \{ \bar{U}_{\mf{x}^\star}(t) < \mu^\star\} \\
        & \leq
        \sum_{t=1}^T \mathds{I} 
        \{ U_{\mf{x}^\star}(t) < \mu^\star\}.
    \end{align}
    Hence, we can derive the following:
    \setlength{\belowdisplayskip}{2pt}
    \setlength{\abovedisplayskip}{2pt}%
    \setlength{\jot}{-2pt}
    \begin{align}
        \mb{E}N_T(\mf{x}) 
        & \leq
        \sum_{t=1}^T P
        ( U_{\mf{x}^\star}(t) < \mu^\star ) \\
        & \leq
        \sum_{t=1}^T P
        ( \hat{\mb{E}}_{P^\ast_{\mf{x}^\star},t}Y + \sqrt{\frac{2\ln(t)}{N_t(\mf{x}^\star)}} < \mu^\star )\\
        & \leq
        \sum_{t=1}^T 
        \sum_{N_t(\mf{x}^\star)=1}^t
        P( \hat{\mb{E}}_{P^\ast_{\mf{x}^\star},t}Y + \sqrt{\frac{2\ln(t)}{N_t(\mf{x}^\star)}} < \mu^\star )\label{Eq. ucb hoef 1}\\
        & \leq
        \sum_{t=1}^T 
        \sum_{N_t(\mf{x}^\star)=1}^t
        \delta = 
        \sum_{t=1}^T 
        \frac{1}{t^3} \leq \frac{\pi^2}{6}\label{Eq. ucb hoef 2}
    \end{align}
    where \Cref{Eq. ucb hoef 1} to \Cref{Eq. ucb hoef 2} follows Hoeffding's inequality. Next, \textbf{let} us consider $u_\mf{x} > \mu^\star$, i.e., the prior upper causal bound is greater than the optimal expected reward. Let $l$ be an arbitrary positive integer.
    \setlength{\belowdisplayskip}{2pt}
    \setlength{\abovedisplayskip}{2pt}%
    \setlength{\jot}{-2pt}
    \begin{align}
        N_T(\mf{x}) 
        &= \sum_{t=1}^T \mathds{I} \{ \mf{X}_t = \mf{x}\} \\
        &\leq 
        l + \sum_{t=1}^T \mathds{I} 
        \{ \bar{U}_{\mf{x}^\star}(t) 
        \leq 
        \bar{U}_\mf{x}(t), N_t(\mf{x}) \geq l\}\\
        & \leq
        l + \sum_{t=1}^T \mathds{I} 
        \{ U_{\mf{x}^\star}(t) 
        \leq 
        U_\mf{x}(t), N_t(\mf{x}) \geq l\}\\
        & =
        l + \sum_{t=1}^T \mathds{I} 
        \{ \hat{\mb{E}}_{P^\ast_{\mf{x}^\star},t}Y 
        + \sqrt{\frac{2\ln(t)}{N_t(\mf{x}^\star)}}
        \leq 
        \hat{\mb{E}}_{P^\ast_{\mf{x}},t}Y 
        + \sqrt{\frac{2\ln(t)}{N_t(\mf{x})}}
        , N_t(\mf{x}) \geq l\}\\
        & \leq
        l + \sum_{t=1}^T \sum_{N_t(\mf{x}^\star)=1}^t \sum_{N_t(\mf{x})=l}^t \mathds{I} 
        \{ \hat{\mb{E}}_{P^\ast_{\mf{x}^\star},t}Y 
        + \sqrt{\frac{2\ln(t)}{N_t(\mf{x}^\star)}}
        \leq 
        \hat{\mb{E}}_{P^\ast_{\mf{x}},t}Y 
        + \sqrt{\frac{2\ln(t)}{N_t(\mf{x})}}\}
    \end{align}
    The last event implies that at least one of the following events occur: 
    \begin{align}
        &\hat{\mb{E}}_{P^\ast_{\mf{x}^\star},t}Y 
        + \sqrt{\frac{2\ln(t)}{N_t(\mf{x}^\star)}} \leq 
        \mb{E}_{P^\ast_{\mf{x}^\star}}Y \label{eq: evert1}\\
        & \hat{\mb{E}}_{P^\ast_{\mf{x}},t}Y 
        - \sqrt{\frac{2\ln(t)}{N_t(\mf{x})}} \geq 
        \mb{E}_{P^\ast_{\mf{x}}}Y \label{eq: evert2}\\
        & \mb{E}_{P^\ast_{\mf{x}^\star}}Y 
        < 
        \mb{E}_{P^\ast_{\mf{x}}}Y
        + 2 \sqrt{\frac{2\ln(t)}{N_t(\mf{x})}}\label{eq: evert3}
    \end{align}
    Remaining part follows the proof of Theorem 1 in \citet{auer2002finite}. We bound the probability of events \Cref{eq: evert1,eq: evert2} using Hoeffding's inequality \citep{hoeffding1994probability}. \Cref{eq: evert3} does not appear 
    \begin{align}
        \mb{E}_{P^\ast_{\mf{x}^\star}}Y  - \mb{E}_{P^\ast_{\mf{x}}}Y
        - 2 \sqrt{\frac{2\ln(t)}{N_t(\mf{x})}}
        = 
        \Delta_{\mf{x}} - 2 \sqrt{\frac{2\ln(t)}{N_t(\mf{x})}} 
        \geq 
        0
    \end{align}
    for $N_t(\mf{x}) \geq \frac{8\ln t}{
    (\Delta_{\mf{x}})^2}$. Therefore, with $l \geq \left\lceil \frac{8\ln t}{
    (\Delta_{\mf{x}})^2} \right\rceil$, we get 
    \begin{align}
       \mb{E} N_T(\mf{x}) 
        & \leq
        \left\lceil \frac{8\ln t}{
    (\Delta_{\mf{x}})^2} \right \rceil + \sum_{t=1}^T \sum_{N_t(\mf{x}^\star)=1}^t \sum_{N_t(\mf{x})=\left\lceil {8\ln t} \slash {
    (\Delta_{\mf{x}})^2} \right \rceil}^t 
    2t^{-4}\leq 
    \frac{8\ln t}{
    (\Delta_{\mf{x}})^2}
    + 1 + \frac{\pi^2}{3}
    \end{align}
    which concludes the proof.
\end{proof}

\section{Details on Experimental Settings}\label{sec: app experimental details}
We compare \textsc{trUCB} (Alg.~\ref{alg: trUCB}) against standard UCB over all combinations of arms (UCB), defined as $\{\mf{x} \in \Omega_{\mf{X}} \mid \mf{X} \in 2^{\mf{V} \setminus \{Y\} \setminus \mf{N}^\ast}\}$ and over POMISs (\textsc{poUCB}), defined as $\{\mf{x} \in \Omega_{\mf{X}} \mid \mf{X} \in\mb{P}_{\mc{G},Y}^{\mf{N}^\ast}\}$. Our baseline comparison is with \textsc{poUCB}, to ensure a fair evaluation of transportability performance. The number of trials is set to 50,000, which is sufficient to observe the performance differences. Each simulation is repeated 1,000 times to ensure consistent results. The experiments were conducted on a Linux server equipped with an Intel Xeon Gold 5317 processor running at 3.0 GHz and 64 GB of RAM. No GPUs were used during the simulations. 

\subsection{Detailed Explanations of the Working Examples}
We provides details on the workings of the \textsc{trUCB} algorithm (Alg.~\ref{alg: trUCB}) and the specific SCMs used in all bandit instances presented in the experiments (\Cref{sec: experiments}). We denote the exclusive-or operation by $\oplus$, and use $\Bern$ to represent a Bernoulli distribution. We \textit{randomly} generate structural functions $\mf{F}$ using binary logical operations ($\land$, $\lor$, $\oplus$, $\lnot$), and the parameters of the exogenous variable distributions are also \textit{randomly} selected. The following table (\Cref{table: experiment}) summarizes our simulation results.  

\begin{table}[!t]
        \centering
        \scriptsize
        \begin{tabular}{l@{\hskip 4pt}c@{\hskip 4pt}c@{\hskip 4pt}c}
            \toprule
              & Task 1 & Task 2 & Task 3 \\

              Total trials & 50k & 50k & 50k \\

            \midrule
            \textsc{trUCB}
            & $\boldsymbol{47.94} \pm 9.60$ {\color{betterred}$(36.83\%)$}        
            & $\boldsymbol{85.25} \pm 113.3$ {\color{betterred}$(38.62\%)$}        
            & $\boldsymbol{95.69} \pm 7.47$ {\color{betterred}$(55.11\%)$}        
            \\
            \textsc{poUCB}
            & $130.16 \pm 7.33$
            & $220.75 \pm 117.82$
            & $173.62 \pm 7.56$
            \\
            \textsc{UCB}
            & $397.03 \pm 12.33$
            & $220.75 \pm 117.82$ (= CR of \textsc{poUCB})
            & $214.42 \pm 7.76$
            \\
            \bottomrule
        \end{tabular}
        \vspace{1em}
        \caption{Mean and standard deviation of cumulative regret over 1,000 repeated simulations. The percentages (red) represent the ratio $\frac{\text{CR for \textsc{trUCB}}}{\text{CR for \textsc{poUCB}}} \times 100 (\%)$. }
        \label{table: experiment}
        \vspace{-1em}
\end{table}

\paragraph{Task 1.} The bandit instance is associated with an SCM $\mc{M}$ where 

\begin{align}
    \mc{M}
    = \begin{cases}
            \mf{U} &= \{U_A,U_B,U_Y, U_{AY}\}\\
            \mf{V} &= \{A,B,C,Y\} \\
            \mf{F} &= 
            \begin{cases}
                f_{A} = u_A \lor (b \oplus u_{AY}), f_{B} = u_B, \\
                f_Y = (1- b) \lor ((u_{AY} \lor b) \oplus a)\}
            \end{cases}\\
            P(\mf{U}) &= 
            \begin{cases}
                U_A \sim \Bern(0.33), U_B \sim \Bern(0.38),\\
                U_Y \sim \Bern(0.41), U_{AY} \sim \Bern(0.75).
            \end{cases}
    \end{cases}
\end{align}

The decision maker has an experimental prior $\mb{Z}^1 = \{\{B\}\}$ from $\Delta^1 = \{A\}$. The action space without prior information corresponds to $\mb{P}_{\mc{G},Y}^{\mf{N}^\ast = \emptyset} = \{\{B\}, \{A,B\}\}$, which corresponds to the action space of the baseline algorithm \textsc{poUCB} (i.e., the initialized target action space). In this setting, we observe that $\mb{E}_{P^\ast_b} Y$ cannot be bounded from the source $\pi^1$, since $\{A\} \cap \Delta^1 \neq \emptyset$. In contrast, $\mb{E}_{P^\ast_{a,b}} Y $ can be bounded as $\sum_{y} y  P^1_b(a,y) \leq \mb{E}_{P^\ast_{a,b}} Y \leq \sum_{y} y  \{P^1_b(a,y) + 1 - P^1_b(a) \}$. Using these expressions, the decision maker can estimate causal bounds for four actions corresponding to the POMISs: $do(A=0,B=0): [0.1675, 1]$, $ do(A=0,B=1): [0.2965, 0.7940]$, $ do(A=1,B=0): [0.8325, 1]$ and $ do(A=1,B=1): [0.2935, 0.7960]$. The algorithm computes dominance bounds as $\ell^\star = 0.8325$ and $u^\star = \infty$. Here, $u^\star = \infty$ arises because the causal bounds for $do(b)$ are $[0,\infty)$, i.e., no information can be transported for that intervention. Among the four actions, the upper causal bounds for $do(A = 0,B = 1)$ and $do(A = 1,B = 1)$ are lower than $\ell^\star$, leading them to be excluded by \textsc{trUCB}. Accordingly, the algorithm begins the online interaction with the final action space $\mc{I}^\ast$ excluding these two actions. We observe that the mean cumulative regret at the final trial is $\boldsymbol{47.94}~(\pm~9.60)$ for \textsc{trUCB} and $\boldsymbol{130.16}~(\pm~7.33)$ for \textsc{poUCB}, which is $\boldsymbol{36.83\%}$ of the latter. 

\paragraph{Task 2.} The bandit instance is associated an SCM $\mc{M}$ where 

\begin{align}
    \mc{M}
    = \begin{cases}
            \mf{U} &= \{U_A,U_B,U_C,U_Y, U_{BC}, U_{BY}\}\\
            \mf{V} &= \{A,B,C,Y\} \\
            \mf{F} &= 
            \begin{cases}
                f_{A} = u_{A}, f_{B} = (u_{BY}) \land u_{BC}) \oplus u_B), \\
                f_C = b \oplus (u_{BC} \land ((1 - u_C) \oplus a)),\\
                f_Y = c \oplus (u_{BY} \land ((1 - u_Y) \oplus a))
            \end{cases}\\
            P(\mf{U}) &= 
            \begin{cases}
                U_A \sim \Bern(0.28), U_B \sim \Bern(0.42), U_C \sim \Bern(0.52)\\
                U_Y \sim \Bern(0.47), U_{BC} \sim \Bern(0.62), U_{BY} \sim \Bern(0.49).
            \end{cases}
    \end{cases}
\end{align}

An agent has priors from two environments with potential discrepancies $\Delta^1 = \{A\}$ and $\Delta^2 = \{B\}$. The prior from the first source is observational $\mb{Z}^1 = \{\emptyset\}$, and from the second source is experimental $\mb{Z}^2 = \{\{C\}\}$. The non-manipulable variable constraint is $\mf{N}^\ast = \{A,C\}$, and the corresponding initial action space is $\mb{P}_{\mc{G},Y}^{\mf{N}^\ast} = \{\emptyset, \{B\}\}$. The action $do(\emptyset)$ is transportable as $\mb{E}_{P^\ast_{\emptyset}}Y = \sum_{a,b,c,y}yQ^2[A]Q^1[B,C,Y] = \sum_{a,b,c,y}y P^2_c(a)P^1(b,c,y \mid a)  = 0.4844$. On the other hand, $do(b)$ is \textit{not} transportable, since $Q^\ast[C]$ is \textit{non}-transportable from the sources: 
$P^\ast_{b}(y) = \sum_{a,c}{\color{NavyBlue}Q^2[A]Q^2[Y]}Q^\ast[C]$. Since $\Delta^1 \cap \{C\} = \emptyset$ and according to \Cref{thm: causal bound}, $Q^\ast[C]$ is bounded by $Q^1[B,C]$, which leads to $\ell_b = \sum_{a, c, y}y {\color{NavyBlue} P^2_{c}(a)P^2_{c}(y | a)}P^1(c | a,b)P^1(b)$. This expression yields $\ell_{do(B{=}0)} = 0.2097$ and $\ell_{do(B{=}1)} = 0.2752$. The upper causal bound for $do(b)$ is given by $\sum_y y \min\{1,\sum_{a, c} {\color{NavyBlue} P^2_{c}(a)P^2_{c}(y | a)}\{P^1(b) P^1(c | a,b) + 1 - P^1(b)\}\}$. We thus have $u_{do(B{=}0)} = 0.6783$ and $u_{do(B{=}1)} = 0.8066$. We now proceed to compute the dominance bounds. The set of MISs with respect to $\langle \mc{G}, Y, \mf{N}^\ast \rangle$ is $\{\emptyset,\{B\}\}$. Since $\ell_\emptyset = 0.4844$ and $\ell_{b^\ast} = 0.6783$, we have\begin{wrapfigure}[11]{r}{0.4\textwidth}


    \vspace{-1em}
    \begin{minipage}[b]{0.4\textwidth}\centering
        \begin{tikzpicture}[x=2.8cm,y=1.5cm,>={Latex[width=1.4mm,length=1.7mm]},
        font=\sffamily\sansmath\tiny,
        line width=0.19mm,
       every node/.style={scale=1},
        RR/.style={draw,circle,inner sep=0mm, minimum size=4.5mm,font=\sffamily\tiny}]\centering
        \node[align=center] (C) at (0,0) {
          {\rm \scriptsize (c)} $\mf{N} = \{C\}$\\
          $\mb{P}_{\mc{G},Y}^{\mf{N}} = \{\{A\},\{A,B\}\}$
        };
        
        \node[align = center] (E) at (1,0) {
          {\rm \scriptsize (d)} $\mf{N} = \emptyset$\\
          $\mb{P}_{\mc{G},Y}^{\mf{N}} = \{\{A\},\{A,C\}\}$
        };
        
        \node[align = center] (A) at (1,1) {
          {\rm \scriptsize (b)} $\mf{N} = \{A\}$\\
          $\mb{P}_{\mc{G},Y}^{\mf{N}} = \{\emptyset,\{B\},$\\
          $\{C\}\}$
        };
        
        \node[align = center] (AC) at (0,1) {
          {\rm \scriptsize (a)} $\mf{N}^\ast = \{A,C\}$\\
          $\mb{P}_{\mc{G},Y}^{\mf{N}^\ast} = \{\emptyset,\{B\}\}$\\\,
        };
        
        \draw[->] (AC) -- (A);
        \draw[->,shorten <=-2mm] (AC) -- (C);
        
        \draw[->,shorten <=-4mm,shorten >=-3mm] (AC) -- (E);

        \draw[->] (A) -- (E);
        \draw[->] (C) -- (E);

        \end{tikzpicture}
    \end{minipage}
    \caption{Hierarchical relationships between POMISs under different constraints. Arrows indicate the direction of dominance relations.}
    \label{wrapfig: app_hierarchy}
\end{wrapfigure}
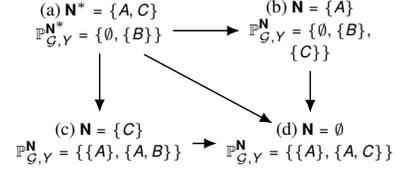the dominance lower bound is given by $\ell^\star = 0.4844$. To compute the upper dominance bound $u^\star$, the algorithm (executed via the subroutine \textsc{udb} in Alg.~\ref{alg: UDB}) initializes with $0.8066$ (i.e., the upper causal bound $u_{do(B=1)}$) and traverses weaker constraints (see \Cref{wrapfig: app_hierarchy}): (b) For $\mf{N} = \{A\}$, the bound remains $0.8066$. (c) For $\mf{N} = \{C\}$, the algorithm returns $0.8070$, which corresponds to the upper causal bound $u_{do(A{=}0,B{=}1)}$. (d) For $\mf{N} = \emptyset$, the algorithm returns $0.7697$, which corresponds to the expected reward of the transportable action $do(A=1,C=1)$---$\mb{E}_{P^\ast_{do(A{=}1,C{=}1)}}Y = \sum_y yP^2_{do(C{=}1)}(y | A{=}1) = 0.7697$. Consequently, the final upper dominance bound is set to $u^\star = 0.7697$. Since $u_{do(B=1)} {=} 0.8066 > 0.7697 {=} u^\star$, the final transport upper bound is updated to the value of $u^\star=0.7697$. The resulting transport bounds for $do(B{=}0)$ and $do(B{=}1)$ are $[0.2097, 0.6783]$ and $[0.2752, 0.7697]$, respectively. Although the size of the action space remains unchanged (i.e., no action is removed from $\mc{I}^\ast$, implying that the action spaces of all three algorithms---\textsc{UCB}, \textsc{poUCB} and \textsc{trUCB}---are identical), we observe that accounting for the transport bounds improves performance, with final mean regrets of $\boldsymbol{220.75}~(\pm~117.82)$ for \textsc{poUCB} and $\boldsymbol{85.25}~(\pm~113.3)$ for \textsc{trUCB}, corresponding to $\boldsymbol{38.62\%}$ of the baseline algorithm \textsc{poUCB}. 

\paragraph{Task 3.} The bandit instance is associated an SCM $\mc{M}$ where 
 
\begin{align}
    \mc{M}
    = \begin{cases}
            \mf{U} &= \{U_R,U_T,U_W,U_X,U_Z,U_Y, U_{RW},U_{RY}, U_{XY}, U_{WX}, U_{WZ}\}\\
            \mf{V} &= \{R,T,W,X,Z,Y\} \\
            \mf{F} &= 
            \begin{cases}
                f_{R} = u_R \lor ((1 - u_{RW}) \land (t \land u_{RY})), f_T = u_T\\
                f_W = u_W \lor ((1 - u_{WX}) \land (u_{RW} \land u_{WZ})), \\
                f_X = w \lor ((1 - u_{WX}) \land (u_{XY} \land u_X)), \\
                f_Z = z \lor ((1 - u_{XY}) \land (u_{RY} \land u_Y)), \\
                f_Y = x \lor ((1 - u_{WZ}) \land (r \land u_Z))
            \end{cases}\\
            P(\mf{U}) &= 
            \begin{cases}
                U_R\ \sim \Bern(0.53), U_T \sim \Bern(0.63), U_W\sim \Bern(0.38), \\
                U_X\sim \Bern(0.27), U_Z\sim \Bern(0.4), U_Y\sim \Bern(0.26), \\
                U_{RW}\sim \Bern(0.52), U_{RY}\sim \Bern(0.63), U_{XY}\sim \Bern(0.79), \\
                U_{WX}\sim \Bern(0.74), U_{WZ}\sim \Bern(0.31).
            \end{cases}
    \end{cases}
\end{align}
 
An agent has access to priors from two environments with potential discrepancies $\Delta^1 = \{R\}$ and $\Delta^2 = \{T\}$, with priors $\mb{Z}^1 = \{\emptyset, \{Z\}\}$ and $\mb{Z}^2 = \{\{Z\}\}$, and constraint $\mf{N}^\ast = \{T,W\}$. The initial action space is $\mb{P}_{\mc{G},Y}^{\mf{N}^\ast} = \{\emptyset, \{R\}, \{X\}, \{Z\}\}$ where $\mb{E}_{P^\ast_{\emptyset}}Y = \sum_{r, t, w, x, z,y} y P^1(r,x,y,z | t,w) P^1(w) P^2_{z}(t)$ and $\mb{E}_{P^\ast_{z}}Y = \sum_y yP^1_{z}(y) = \sum_y yP^2_{z}(y)$ are transportable while $\mb{E}_{P^\ast_{x}}Y = \sum_{r,t,z,y} y {\color{NavyBlue}Q^1[R,Y]Q^2[T]}Q^\ast[Z]$ and $\mb{E}_{P^\ast_{r}}Y = \sum_{w, x, z,y} y Q^\ast[W,X,Z,Y]$ are \textit{not}. According to \Cref{thm: causal bound}, we have that $Q^\ast[Z]$ can be bounded by $Q^1[W,R,X,Z,Y] = P^1(w,r,x,z,y\mid t)$, leading to $\sum_{r,t,z,y} y {\color{NavyBlue}P^1_{z}(r,y | t)P^2_{z}(t) }P^1(r,w,x,z | t) \leq \mb{E}_{P^\ast_{x}}Y \leq \sum_y y \min\{1,\sum_{r, t, z} {\color{NavyBlue}P^1_{z}(r,y | t)P^2_{z}(t)}(P^1(r,w,x,z | t) + 1 - P^1(r,w,x | t))\}$. Similarly, we can derive the causal bounds for $do(r)$ using $Q^1[W,R,X,Z,Y]$ as $\sum_{w, x, z} P^1(r,x,y,z | t,w) P^1(w) \leq \mb{E}_{P^\ast_{r}}Y \leq \sum_y y \min \{1,\sum_{w, x, z} P^1(r,x,y,z | t,w) P^1(w) + 1 - P^1(r,x,z | t,w) P^1(w))\}$. In this setting, the transportable target POMIS action yields $\mb{E}_{P^\ast_{do(Z=1)}}Y = 1$, resulting in $\ell^\star = u^\star = 1$. Furthermore, we find the upper causal bounds $u_{do(\emptyset)} = 0.5514$, $u_{do(X=0)} = 0.7901$ and $u_{do(Z=0)} = 0.034$---these three upper causal bounds are lower than $\ell^\star$. Hence, the corresponding actions are eliminated from $\mc{I}^\ast$ by the algorithm. We observe mean cumulative regrets of $\boldsymbol{173.62 }~(\pm~7.56)$ for \textsc{poUCB} and $\boldsymbol{95.69}~(\pm~7.47)$ for \textsc{trUCB}, corresponding to $\boldsymbol{55.11\%}$ of regret ratio. 

\section{Discussions}\label{sec: app discussion}
\paragraph{Misspecification.}
In the transportability literature, assuming access to true selection diagram is a standard modeling practice. As noted by \citet{bareinboim2016causal},
\textit{if no knowledge about commonalities and disparities across environments is available, transportability cannot be justified}. That said, some degree of misspecification can be tolerated without invalidating the performance guarantees---particularly when the assumed causal diagram or selection diagram forms \textit{super-model} of the true environment. A similar discussion was presented by \citet{bellot2023transportability}.  

For example, suppose both the source and target environments share the causal diagram $\mathcal{G} = \langle \{X,Y\}, \{X \to Y\} \rangle$. Given source data $P(x,y)$ and the graph structure, the interventional distribution $P_x(y) = P(y \mid x)$ is identifiable. However, if we are unsure about the presence of an unobserved confounder, we can conservatively posit a super-model $\mathcal{G}' = \langle \{X,Y\}, \{X \to Y, X \leftrightarrow Y\} \rangle$. Under this model, while point identification fails, the interventional distribution $P_{x}(y)$ is bounded within the interval $[P(x,y), P(x,y) + 1 - P(x)]$, which still contains the true value. The same reasoning applies to selection diagrams. The selection nodes only indicate potential discrepancies between environments. If a researcher is uncertain about whether a mechanism differs across environments, they may still conduct valid inference by conservatively assuming the presence of a discrepancy. Such conservatism does not harm transportability guarantees. Importantly, this conservative modeling increases the number of POMISs. For instance, a POMIS under $\mathcal{G}$ is $\{\{X\}\}$, while under $\mathcal{G}'$ it becomes $\{\emptyset, \{X\}\}$, which covers the true POMIS. This leads to less informative but still correct inferences, outperforming methods that ignore structural information altogether. 

In contrast, misspecifying in the opposite direction (i.e., failing to model a discrepancy that does exist in the target environment) can lead to incorrect inferences. This reflects a fundamental asymmetry: being conservative preserves soundness, but missing edges or selection nodes can violate correctness.

\paragraph{Parametric approach.}
One might be concerned that while the algorithm (Alg.~\ref{alg: trUCB}) effectively uses prior knowledge to eliminate non-optimal actions before learning begins, it then switches to a traditional UCB approach that ignores additional observations available during each round. Indeed, there exists a rich body of research that incorporates prior knowledge to iteratively update parameters of SCMs under graphical constraints \citep{zhang2022online,bellot2023transportability,wei2023approximate,jalaldoust2024partial}
in online learning. However, such approaches often rely on optimization-based approaches---such as \textit{canonical SCM} \citep{zhang2022partial} or \textit{neural causal models} (NCMs) \citep{xia2021causal} that assume full parameterization, which results in high computational overhead. This approach is infeasible for larger or denser causal diagrams, and exploiting them effectively remains an open problem \citep{elahi2024partial}. Moreover, these methods are constrained to categorical settings, which makes them inapplicable to continuous domains, such as those encountered in causal Bayesian optimization (\textsc{cbo}; \citet{aglietti2020causal} and muti-outcome variant (\textsc{mo-cbo}; \citet{bhatija2025multi}), where POMISs are also leveraged for structural pruning in continuous action spaces. 

In contrast, our approach focuses on leveraging structural knowledge offline, before any online interaction, and without requiring parameterization or any strong assumptions beyond a given graphical structure and sources. This distinction allows us to scale to settings where parameter learning is infeasible or computationally prohibitive (dense graph or continuous domain), while still retaining provable regret guarantees through structure-informed pruning and closed-form bounding. 

\paragraph{Future work.} Recently, research on transportability theory in practical settings has been increasing \citep{jung2024efficient,jalaldoust2024partial,jalaldoust2024transportable}. Beyond structural causal bandits, we believe that transportability-based decision making will offer substantial practical value when integrated with causal reinforcement learning \citep{zhang2022online,hwang2024fine,bareinboim2024introduction}, rehearsal learning \citep{qin2023rehearsal,qin2025gradient,du2024avoiding,duen2025abling,Lue2025avoding}, and sequential planning \citep{pearl1995SBD,jung2024complete}.

\section{Limitations}\label{sec: app limitation}
In this section, we discuss limitations of our work and outline promising directions for future research. 

\paragraph{Modeling bandit instances in the form of SCMs.} Structural Causal Models (SCMs) are a versatile and expressive framework that provides a principled way to represent and reason about causal relationships. Their generality makes them applicable across a wide range of domains. However, SCMs come with certain limitations, such as the assumption of a well-defined set of variables and a fixed causal structure, which may not adequately capture the complexity of dynamic, high-dimensional, or partially observed systems. Nonetheless, our work addresses a fundamental problem within the SCM framework. We believe it provides a solid foundation for future research, such as extending causal bandits to more complex or less structured environments.

\paragraph{Known causal diagram.}  
We make the standard assumption that the deployment learner has access to the underlying causal diagram. While knowledge of the causal structure can greatly enhance decision-making, this requirement may limit the broader applicability of the proposed approach. In practice, several techniques---such as causal discovery methods or the use of ancestral graphs as plausible explanations---can help alleviate this issue. However, these techniques typically require substantial domain knowledge, and thus, the assumption remains a key limitation of our framework. 

While a collective selection diagram provides a principled and interpretable framework for analyzing environment shifts and transferring information across environments, it arguably restricts the analysis to a narrow class of problems where common (super) causal diagrams should be explicitly defined. Consequently, the applicability of this framework may be limited in settings where environmental changes substantially alter the underlying graph structure.  

\paragraph{Tightness of causal bounds.}
As discussed in \Cref{sec: app partial trans}, the upper bound computed by the algorithm \textsc{paTR} (Alg.~\ref{alg: paTR}) can, in some cases, exceed one. Moreover, we do not guarantee the \textit{tightness} of our causal bounds; that is, we cannot ensure the existence of an SCM under which the causal effect exactly equals either $\ell$ or $u$. Characterizing conditions under which lower and upper bounds are attainable by some SCM would enhance the interpretability and reliability of our framework. Investigating tighter bounds and formally establishing their tightness remains an important direction for future work.

\paragraph{Sufficient prior data from source environments.}  
Our method relies on the availability of prior data to construct dominance and causal bounds for guiding exploration. However, when the prior data is insufficient or biased, the resulting bounds may become inaccurate. In particular, inaccurate bounds may fail to include the true expected reward of certain actions. As a result, the agent may prematurely eliminate potentially optimal actions from exploration, leading to suboptimal performance. Moreover, our current framework focuses on the transportability of causal quantities, rather than on improving estimation quality under noisy or limited prior data. Developing robust algorithms that can explicitly account for uncertainty in the prior and avoid overconfident pruning remains an important direction for future work. 


\section*{Impact Statement}
This work addresses a structured causal bandit framework that leverages prior knowledge from heterogeneous environments through transportability theory in the causal inference literature. This approach has potential applications in real-world settings where experimentation is costly or limited, such as personalized healthcare, adaptive education, and resource-constrained recommendation systems. In these domains, improper specification of causal structures may lead to misleading conclusions and biased decisions. Therefore, careful validation and domain-specific causal modeling are essential before deployment in high-stakes environments. 

\end{document}